\documentclass[11pt]{article}
\usepackage[margin=1.1in]{geometry}

\usepackage{amsmath, amssymb, amsfonts, amsthm, mathrsfs, amsbsy}
\usepackage{graphicx, epsfig}
\usepackage{tikz}
\usepackage[export]{adjustbox}
\usepackage{subcaption}

\usepackage{algorithm, algpseudocode}

\usepackage{array}
\usepackage{lipsum}

\DeclareMathOperator*{\argmin}{arg\,min}

\newcommand{\KL}{\mathrm{D}_\mathrm{KL}}
\newcommand{\tphi}{\tilde{\phi}}

\newcommand{\prox}{\text{prox}}

\newtheorem{theorem}{Theorem}
\newtheorem{prop}[theorem]{Proposition}
\newtheorem{lemma}[theorem]{Lemma}

\usepackage{enumitem}
\usepackage{xcolor}
\usepackage{hyperref}
 
\title{Sparse Transformer Architectures via \\Regularized Wasserstein Proximal Operator with $L_1$ Prior}
\author{
  Fuqun Han\thanks{Department of Mathematics, University of California, Los Angeles, CA, USA. \texttt{fqhan@math.ucla.edu}.  F. Han is partially supported by AFOSR YIP award No. FA9550-23-1-0087.}
  \and
  Stanley Osher\thanks{Department of Mathematics, University of California, Los Angeles, CA, USA. \texttt{sjo@math.ucla.edu}. 
   S. Osher is partially supported by DARPA under grant HR00112590074, NSF under grants 2208272 and 1554564, AFOSR under MURI grant N00014-20-1-278, and ARO under grant W911NF-24-1-015.
}
  \and
  Wuchen Li\thanks{Department of Mathematics, University of South Carolina, Columbia, SC, USA. \texttt{wuchen@mailbox.sc.edu}. W. Li is partially supported by AFOSR YIP award No. FA9550-23-1-0087, NSF DMS-2245097, and NSF RTG: 2038080.}
}
\date{}
 
\begin{document}
\maketitle
\begin{abstract}
In this work, we propose a sparse transformer architecture that incorporates prior information about the underlying data distribution directly into the transformer structure of the neural network. The design of the model is motivated by a special optimal transport problem, namely the regularized Wasserstein proximal operator, which admits a closed-form solution and turns out to be a special representation of transformer architectures. Compared with classical flow-based models, the proposed approach improves the convexity properties of the optimization problem and promotes sparsity in the generated samples. Through both theoretical analysis and numerical experiments, including applications in generative modeling and Bayesian inverse problems, we demonstrate that the sparse transformer achieves higher accuracy and faster convergence to the target distribution than classical neural ODE–based methods.
\end{abstract}

\section{Introduction}

Modern generative models, such as neural ordinary differential equations (neural ODEs) \cite{neural_ode}, transformers \cite{vaswani2017attention}, and diffusion models \cite{song2020score}, have demonstrated remarkable ability to learn and generate samples from complex, high-dimensional probability distributions.
These architectures have achieved broad success in scientific computing, image processing, and data science, offering scalable frameworks for data-driven modeling.
However, training and sampling in such spaces remain expensive and highly sensitive to architectural and optimization choices.

Despite these advances, the curse of dimensionality continues to present a fundamental challenge in many real-world applications.
Fortunately, numerous problems in scientific computing exhibit intrinsic structures, such as sparsity, low-rank representations, or approximate invariances, that can be interpreted as prior information about the underlying data or operators.
Leveraging such priors within generative models offers a promising avenue to improve both computational efficiency and generalization.

A classical way to incorporate prior information, such as sparsity or piecewise regularity, is through Bayesian modeling, where the posterior combines a prior distribution encoding structural knowledge with a likelihood function derived from observations. Examples include the $L_1$ prior in compressive sensing \cite{candes2006robust} and the total-variation prior in image reconstruction \cite{rudin1992nonlinear}.
Yet embedding such priors into modern neural networks is challenging: most architectures are general-purpose approximators, and sampling from complex priors (e.g., TV) is computationally prohibitive.

To address these challenges, we embed prior information into the neural architecture through the regularized Wasserstein proximal operator (RWPO).
Classical transformer and diffusion models are largely data-driven, learning attention or score functions from data with limited connection to underlying transport dynamics.
In contrast, our approach originates from the backward regularized Wasserstein proximal (BRWP) sampling algorithm \cite{han2025splitting}, developed to stabilize deterministic probability flows via alternating transport and semi-implicit proximal steps.
We scale up this sampling framework by embedding each BRWP update into a network layer: the transport step is parameterized by a learnable drift potential, and the proximal step becomes a transformer-type interaction derived from the RWPO kernel.

To formalize this idea, consider the problem of approximating a target distribution $\rho^* \in \mathbb{P}_2(\mathbb{R}^d)$, representing a probability density with finite second moment, given training data drawn from $\rho^*$. For a time horizon $T>0$, we define the variational problem
\begin{equation}
\label{rho_T_opt}
\rho_T = 
\arg\min_{\rho_T}\min_{\tphi_t}
\Bigg\{ 
\KL(\rho^*\|\rho_T)
+\frac{1}{2}
\int_0^T\!\!\int 
\|\nabla\tphi_t\|_2^2\,\rho_t\,dx\,dt
\Bigg\}\,,
\end{equation}
where $\KL(\rho^*\|\rho_T)$ denotes the Kullback–Leibler (KL) divergence between the data distribution $\rho^*$ and the model’s terminal distribution $\rho_T$.  
The inner minimization describes the most energy-efficient evolution of densities connecting $\rho_0$ and $\rho_T$ in the sense of Wasserstein-2 distance.  
The corresponding constrained dynamical system yields the Fokker–Planck equation, whose particle-level evolution is given by
\begin{equation}
\label{FPK_particle}
 dX_t = \nabla\tilde{\phi}_t(X_t) - \beta^{-1}\nabla\log\rho_t(X_t)\,,
\end{equation}
where $\tilde{\phi}_t:\mathbb{R}^d\to\mathbb{R}$ denotes the trainable drift potential, $X_t\sim\rho_t$, $\rho_t(x)=\rho(t,x)$ for $t\in[0,T]$, and $\beta>0$ controls the diffusion strength.
In this formulation, the particle dynamics generalize those of a neural ODE~\cite{neural_ode}. The drift term $\nabla\tilde{\phi}_t$ serves as the neural network–parameterized drift potential, while the additional score term $-\beta^{-1}\nabla\log\rho_t$ introduces a diffusion-driven regularization absent in classical neural ODEs. This score-based correction couples the flow with the Fokker–Planck evolution of the underlying density, acting as a diffusion-induced force that repels particles from high-density regions, prevents collapse, and balances the attractive drift $\nabla\tilde{\phi}_t$ toward the target distribution.

In high-dimensional settings, estimating the score function$\nabla\log\rho_t$ directly is computationally expensive and often unstable.  
The RWPO framework circumvents this difficulty by providing a first-order approximation of the density evolution in \eqref{FPK_particle}, while retaining a closed-form representation that approximates the score function.

More concretely, assuming $\tphi_t = \phi_t - \psi$, we take $\phi_t$ to be represented by a residual neural network and  
$\psi(x) = \lambda \|x\|_1$
to encode prior information on data sparsity in a time-independent potential.  
Denote the tokens at time $t_k$ as $X^k = \{x_j^k\}_{j=1}^N$.  
When the step size $h$ is small, by a first-order approximation of the RWPO kernel, the tokens (samples) evolve according to
\begin{equation}
\label{eqn_particle_intro}
\begin{cases}
x_j^{k+1/2} = x_j^k + h\nabla\phi_k(x_j^k)\,, \\
x_j^{k+1} = x_j^{k+1/2} + \tfrac{1}{2}\!\left[\,S_{\lambda h}(x_j^{k+1/2}) 
- \sum_{\ell=1}^N \operatorname{softmax}\!\left(U(x_j^{k+1/2},X^{k+1/2})\right)_{\ell}\,x_{\ell}^{k+1/2}\right]\,,
\end{cases}
\end{equation}
where $\phi_k = \phi_{t_k}$.  
Here
\[
S_{\lambda h}(x) = \operatorname{ReLU}(|x|-\lambda h) 
\quad\text{with}\quad 
\operatorname{ReLU}(z) = \max\{0,z\}
\]
is the proximal operator of the $\ell_1$-norm, and the interaction kernel is defined by
\[
U(x,y) := -\frac{\beta}{2}\left(\frac{\|x-y\|_2^2-\|S_{\lambda h}(x)-y\|_2^2}{2h}-\lambda\|S_{\lambda h}(y)\|_1\right), 
\]
with $\operatorname{softmax}(\omega) = \left(\frac{\exp(\omega_j)}{\sum_{\ell=1}^N \exp(\omega_{\ell})}\right)_{1\leq j \leq N}$.

The second update in \eqref{eqn_particle_intro} can be viewed as a transformer-type interaction layer, which we refer to as a sparse transformer layer.  
For comparison, recall that a standard transformer block (cf.~\cite{castin2025unified, geshkovski2023mathematical, sander2022sinkformers}) with value, query, and key matrices $V \in \mathbb{R}^{d\times d}$ and $Q,K \in \mathbb{R}^{m\times d}$ can be written in residual form as
\begin{equation}
\label{general_transformer}
x_j^{k+1} = x_j^{k} + h \sum_{\ell=1}^N 
\operatorname{softmax}\!\left((Qx_j^k \cdot KX^k)\right)_{\ell}\, V x_{\ell}^k\,,
\end{equation}
which can be interpreted as an interacting particle system with kernel $Qx_j^k \!\cdot\! Kx_\ell^k$.  
In contrast, the proposed sparse transformer layer replaces the linear dot-product kernel with a nonlinear RWPO-based kernel where $Qx\cdot Ky$ is replaced by $U(x,y)$.
This substitution preserves the attention-style weighted averaging but introduces a potential-driven interaction derived from optimal-transport dynamics.  
The term $S_{\lambda h}$ enforces sparsity, while the diffusion term enhances stability.  
Consequently, the layer acts as a physics-motivated analogue of transformer attention, embedding score-based diffusion into the network.

Classical approaches to generative modeling often rely on approximating score functions.  
Score matching \cite{hyvarinen2005estimation} and its denoising variant \cite{vincent2011} laid the foundation for score-based generative models and diffusion methods, which learn time-dependent score functions and generate samples via reverse-time dynamics \cite{song2019generative, song2020score}.  
Complementary to these stochastic formulations, continuous-time viewpoints interpret learning and generation through deterministic transport flows.  
Neural ODEs \cite{neural_ode} view residual networks as discretizations of continuous transformations between distributions, while Langevin-type dynamics and their discretizations provide theoretically grounded sampling schemes with convergence guarantees under log-concavity assumptions \cite{roberts1996exponential, durmus2019analysis}.  
These frameworks emphasize two central ingredients in modern generative modeling: score estimation and space-time discretizations of transport dynamics.

Our work departs from score-estimation methods and instead builds on regularized Wasserstein gradient flows.
The Jordan–Kinderlehrer–Otto (JKO) scheme \cite{jko1998} characterizes diffusion as an entropy gradient flow in Wasserstein space, linking sampling, dissipation, and variational regularization.
This viewpoint motivates Wasserstein proximal operators as stable, structure-preserving updates \cite{ambrosio2005gradient, santambrogio2017euclidean} and has recently inspired OT-based attention mechanisms \cite{sander2022sinkformers, geshkovski2023mathematical}.
From a Bayesian perspective, these proximal operators naturally encode priors via penalties such as the Lasso \cite{tibshirani1996regression} or total variation \cite{rudin1992nonlinear}.
Building on these ideas, our framework embeds the RWPO directly into the network, allowing priors to shape token evolution through each layer.
It unifies three principles:
(i) optimal-transport geometry via Wasserstein proximal steps approximating Fokker–Planck evolution;
(ii) Bayesian priors through explicit proximal penalties promoting sparsity; and
(iii) transformer-type interaction arising from the softmax normalization of the RWPO kernel.
The resulting sparse-transformer layer acts as a single proximal step in Wasserstein space with a closed-form attention update that preserves prior-induced structure.

An overview of the remaining sections is as follows.  
Section~\ref{sec_network} introduces the RWPO layer and sparse transformer architecture.  
Section~\ref{sec_analysis} provides the theoretical analysis of convergence and stability for the forward dynamics and optimization problem.  
Section~\ref{sec_NE} demonstrates performance in generative modeling and Bayesian inverse problems.

\section{Proposed network structure and PDE formulation} 
\label{sec_network}
\subsection{Probability flow and particle updates}
\label{sec:prob_flow}

For generative modeling, the central task is to generate new samples from a target distribution $\rho^*$, typically parameterized through a neural network.  
We begin by recalling the KL divergence between two distributions,
\[
\KL(\rho\|\rho^*) = \int \rho \log \tfrac{\rho}{\rho^*}\,dx\,.
\]
The Wasserstein gradient flow of the KL divergence yields the Fokker–Planck equation
\begin{equation}
\label{FPK}
    \partial_t\rho + \nabla\!\cdot(\rho\nabla \tphi) = \beta^{-1}\Delta \rho\,,
\end{equation}
where $\tphi = -\log\rho^*$.  
At the particle level, this corresponds to the probability flow 
\begin{equation}
\label{def_ode}
dX_t = \nabla \tphi(X_t) - \beta^{-1}\nabla \log \rho_t(X_t)\,,
\end{equation}
with $X_t\!\sim\!\rho_t$.  
Under standard smoothness and log-concavity assumptions on $\rho^*$, the dynamics \eqref{def_ode} converge exponentially fast to $\rho^*$ \cite{roberts1996exponential}.  

Equation~\eqref{def_ode} admits two complementary interpretations.  
In sampling, it represents forward stochastic dynamics that asymptotically draw samples from $\rho^*$ when the drift potential is known.  
For generative modeling, the same drift potential can be parameterized by a neural network, transforming a simple base distribution $\rho_0$ into $\rho^*$ through a sequence of deterministic updates.  
This viewpoint connects probability flows with neural network architectures: the drift potential $\phi_t$ acts as a trainable function governing density evolution, and each discretization step corresponds to a layer of the neural network.

In practice, two difficulties arise when discretizing \eqref{def_ode}:  
(i) computing or approximating $\nabla \log \rho_t$ is expensive and often unstable, and  
(ii) the dynamics are highly sensitive to the choice of step size.  
To address these challenges, \cite{han2025splitting} introduced the deterministic BRWP-splitting scheme for sampling from distributions of the form
\[
\rho^*(x) = \tfrac{1}{Z}\exp(\phi(x)-\psi(x)) = \tfrac{1}{Z}\exp(-\tphi(x))\,,
\]
where $\psi$ encodes prior regularization and $Z = \int \exp(-\tphi(x))dx < \infty$ is a normalizing constant.  
The corresponding regularized particle dynamic is
\begin{equation}
\label{def_ode_prior}
dX_t = \nabla \phi_t(X_t) - \nabla \psi(X_t) - \beta^{-1}\nabla \log \rho_t(X_t)\,.
\end{equation}
A forward-Euler discretization with operator splitting \cite{han2025splitting}  yields the following update
\begin{equation}
\label{def_split_particle}
\begin{cases}
x_j^{k+1/2}  = x_j^k + h\nabla \phi_k(x_j^k)\,,\\[3pt]
x_j^{k+1}  = x_j^{k+1/2} - h\nabla \psi(x_j^{k+1/2}) - h\beta^{-1}\nabla \log \rho_{k+1/2}(x_j^{k+1/2})\,,
\end{cases}
\end{equation}
where $X^k = \{x_j^k\}_{j=1}^N$ denotes the ensemble of particles at time step $k$.  
The first step applies the drift potential induced by $\phi_k$, and the second introduces the prior and diffusion corrections.  

Direct evaluation of the score function $\nabla\log\rho_t$ remains challenging in high dimensions.  
To stabilize the dynamics and avoid explicit score estimation, we replace the second step in \eqref{def_split_particle} with a semi-implicit update based on the RWPO.  
The resulting scheme, whose derivation is detailed below, evolves as
\begin{equation}
\label{eqn_particle}
\begin{cases}
x_j^{k+1/2} = x_j^k + h\nabla \phi(x_j^k)\,,\\[3pt]
x_j^{k+1} = x_j^{k+1/2} + \tfrac{1}{2}\!\left[\,\prox_{\psi}^h(x_j^{k+1/2}) 
- \sum_{\ell=1}^N \operatorname{softmax}\!\big(U(x_j^{k+1/2},X^{k+1/2})\big)_{\ell}\,x_{\ell}^{k+1/2}\right]\!,
\end{cases}
\end{equation}
where $\prox_{\psi}^h$ is the proximal operator of $\psi$,
\[
\prox_{\psi}^h(x) := \argmin_y \Big\{\psi(y)+\tfrac{\|x-y\|_2^2}{2h}\Big\},
\quad
U(x,y) := -\tfrac{\beta}{2}\!\left(\tfrac{\|x-y\|_2^2-\|\prox_{\psi}^h(x)-y\|_2^2}{2h}-\psi(\prox_{\psi}^h(y))\right)\!.
\]
This formulation combines closed-form proximal updates with a semi-implicit treatment of the score term, allowing large, stable time steps without explicitly computing $\nabla\log\rho_t$.  
The interaction kernel $U$ arises from a Laplace approximation of the RWPO operator, as described in the next subsection.

In the original BRWP scheme, the flow alternates between a transport step guided by the drift potential and a proximal step that enforces prior regularization, yielding a stable deterministic sampling algorithm.  
Here, we reinterpret each BRWP update as a neural network layer: the transport step becomes a learnable map $\nabla\phi_k$ parameterized by a ResNet, while the proximal step is realized as a transformer-type interaction derived from the RWPO kernel.  
This construction converts the sampling algorithm into a deep, scalable architecture for generative modeling.  
To emphasize the connection with attention mechanisms, the particles $x_j^k$ are referred to as \emph{tokens}. The first update in \eqref{eqn_particle} acts as a learnable transport map, and the second enforces prior-induced regularization with diffusion.

The overall network structure is illustrated in Fig.\,\ref{fig_diagram}, showing how the learned drift potential $\phi$ is coupled with an RWPO-based sparse-transformer layer that enforces sparsity through proximal updates.  
A detailed derivation follows.
\begin{figure}[h]
    \centering
    \includegraphics[width=1\linewidth]{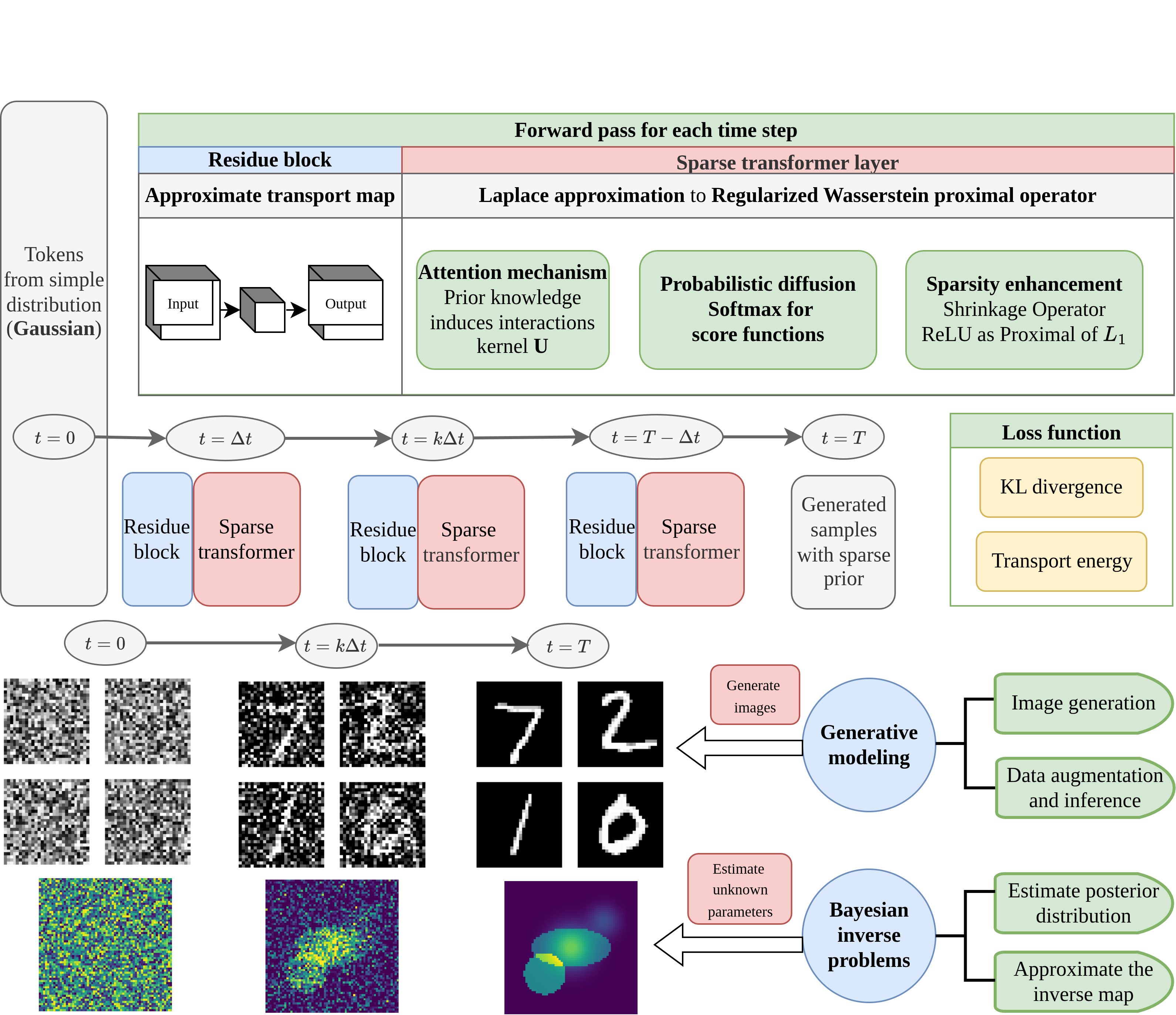}
    \caption{Network structure: the drift potential $\phi$ is approximated by a ResNet, followed by a sparse-transformer (RWPO) layer that enforces sparsity through proximal updates.}
    \label{fig_diagram}
\end{figure}

\subsection{Regularized Wasserstein proximal operator layer}
\label{sec:RWPO_layer}

We now detail the derivation of the system \eqref{eqn_particle} and update rule underlying Fig.\,\ref{fig_diagram}.  
The main numerical difficulty in \eqref{def_split_particle} lies in the unstable approximation of the score function $\nabla\log\rho_t$.  
To mitigate this, we adopt a semi-implicit strategy by evaluating the score at the next time step $\rho_{k+1}$, replacing the second update in \eqref{def_split_particle} with
\begin{equation}
\label{def_split_particle_2}
x_j^{k+1} = x_j^{k+1/2} - h \nabla \psi(x_j^{k+1/2}) - h\beta^{-1}\nabla \log \rho_{k+1}(x_j^{k+1/2})\,.
\end{equation}
Formally, $\rho_{k+1}$ can be characterized through the JKO scheme, which requires solving a high-dimensional variational problem.  
To obtain a tractable approximation, we instead employ the \emph{regularized Wasserstein proximal operator}  \cite{kernel_proximal}, defined as
\begin{equation}
\label{eqn:Wprox}
\mathcal{K}^h_{\psi}(\rho_k)
:=
\argmin_{\rho_h}\,\min_{v,\rho}
\int \psi(x)\,\rho_h(x)\,dx
+\tfrac{1}{2}\int_0^h \!\!\int \|v(x,t)\|_2^2 \rho(x,t)\,dx\,dt,
\end{equation}
subject to the Fokker–Planck equation \eqref{FPK} with $\nabla\tphi$ replaced by $v$, initial condition $\rho(x,0)=\rho_k(x)$, and terminal condition $\rho_h(x)=\rho(x,h)$.  
We assume the minimization admits a solution under standard regularity conditions on $\rho_k$ and $\psi$.

From the analysis in \cite{han2024convergence}, the RWPO provides a second-order accurate approximation of the JKO update:
\[
\mathcal{K}^h_{\psi}(\rho_{k+1/2}) = \rho_{k+1} + \mathcal{O}(h^2)\,.
\]
This justifies replacing $\rho_{k+1}$ in \eqref{def_split_particle_2} with its RWPO approximation, leading to a stable and implementable update.

When $\psi$ is nonsmooth (e.g., $\psi(x)=\lambda\|x\|_1$), we denote by $\partial\psi$ its subdifferential.  
The gradient step then becomes the proximal relation
\[
\frac{x - \prox^h_{\psi}(x)}{h} \in \partial \psi(x)\,,
\]
which ensures consistency with the nonsmooth regularization.

Combining the drift and proximal corrections yields the complete token updates
\begin{equation}
\label{def_split_particle_implicit}
\begin{cases}
x_j^{k+1/2} = x_j^k + h\nabla \phi_k(x_j^k)\,,\\[3pt]
x_j^{k+1}   = \prox_{\psi}^h(x_j^{k+1/2}) - h\beta^{-1}\nabla \log \mathcal{K}^h_{\psi}\rho_{k+1/2}(x_j^{k+1/2})\,,
\end{cases}
\end{equation}
which defines a semi-implicit update integrating prior regularization and the regularized Wasserstein proximal operator.  

\subsection{Sparse transformer with \texorpdfstring{$L_1$}{L1} prior}
\label{sec:L1layer}

We now specialize the RWPO layer to the case where the target distribution exhibits sparsity, a common scenario in Bayesian inverse problems and image reconstruction.  
A classical choice of sparsity-promoting prior is the $L_1$ potential $\psi(x) = \lambda \|x\|_1.$

For this choice, the proximal operator takes the soft-thresholding form
\[
S_{\lambda h}(x_i) = \prox_{\lambda\|\cdot\|_1}^h(x)_i =
\begin{cases}
x_i - \lambda h & \text{if } x_i > \lambda h,\\
0 & \text{if } |x_i| \leq \lambda h,\\
x_i + \lambda h & \text{if } x_i < -\lambda h\,,
\end{cases}
= \mathrm{ReLU}(|x_i|-\lambda h)\,.
\]

Next, we recall a general closed-form expression for the RWPO operator, which serves as the foundation for constructing token-interaction kernels.  
The result follows from applying the Hopf–Cole transform, whose proof can be found in the Appendix.

\begin{lemma}
\label{lem:rwpo}
Let $\psi:\mathbb{R}^d\to\mathbb{R}$ be a convex potential, and let $\rho_k \in \mathbb{P}^2(\mathbb{R}^d)$.  
Then the regularized Wasserstein proximal operator \eqref{eqn:Wprox} admits the integral representation
\begin{equation}
\label{closed_RWPO}
\mathcal{K}^h_{\psi} \rho_k(x) 
= \int_{\mathbb{R}^d}
\frac{\exp\!\left[-\tfrac{\beta}{2}\!\left(\psi(x) + \tfrac{\|x-y\|_2^2}{2h}\right)\right]}
{\displaystyle \int_{\mathbb{R}^d}\!\exp\!\left[-\tfrac{\beta}{2}\!\left(\psi(z) + \tfrac{\|z-y\|_2^2}{2h}\right)\right]dz}\,
\rho_k(y)\,dy\,,
\end{equation}
where the kernel depends on $\psi$, the step size $h>0$, and the inverse temperature constant (diffusion strength) $\beta>0$. 
\end{lemma}

The representation \eqref{closed_RWPO} involves a high-dimensional integral that is generally intractable.  
In practice, we approximate $\rho_k$ by the empirical measure
$\rho_k(x) \approx \tfrac{1}{N}\sum_{j=1}^N \delta_{x^k_j}(x)$
and apply the Laplace approximation to obtain a computable first-order estimate of the score function.

\begin{prop}
\label{prop:approx_score}
Let $\psi(x) = \lambda \|x\|_1$.  
As $h \to 0$, the score function associated with $\mathcal{K}^h_\psi$ satisfies the first-order approximation
\begin{align}
\label{approx_score}
\lim_{h\rightarrow 0}\!\left|\nabla \log \mathcal{K}_{\psi}^{h} \rho_{k+1/2}(x)
-\left[ - \frac{\beta}{2}\!\left( \frac{x - S_{\lambda h}(x)}{h}
+ \frac{\sum_{j=1}^N (x - x_j^{k+1/2}) \exp(U(x,x_j^{k+1/2}))}{h\sum_{j=1}^N \exp(U(x,x_j^{k+1/2}))} \right)\right]\!\right| = 0\,,
\end{align}
where $S_{\lambda h}(x)$ is the soft-thresholding operator and
\[
U(x,y) := -\frac{\beta}{2}\!\left(\frac{\|x-y\|_2^2-\|S_{\lambda h}(x)-y\|_2^2}{2h}-\lambda\|S_{\lambda h}(y)\|_1\right).
\]
Consequently, the second update step in \eqref{def_split_particle_implicit} reduces to
\begin{equation}
\label{approx_score_2}
x_j^{k+1} = x_j^{k+1/2} + \tfrac{1}{2}\!\left[S_{\lambda h}(x_j^{k+1/2})
- \sum_{\ell=1}^N \operatorname{softmax}(U(x_j^{k+1/2},\,X^{k+1/2}))_{\ell}\,x_{\ell}^{k+1/2}\right].
\end{equation}
\end{prop}

\begin{proof}
The result follows by isolating the dominant contribution of the kernel integral in \eqref{closed_RWPO} as $h \to 0$; see \cite{han2025splitting} for details.
\end{proof}

If $\psi$ is $C^{2}$ with a non-degenerate minimizer, the approximation error in \eqref{approx_score} is $\mathcal{O}(h)$ by the standard Laplace method, with the leading correction determined by the local Hessian of $\psi$.  
When $\psi$ is only Lipschitz—as in the $L_{1}$ case—the limit in \eqref{approx_score} still holds because the score function can be expressed as $\nabla \log \mathcal{K}_{\psi}^{h} \rho = (\nabla \mathcal{K}_{\psi}^{h} \rho)/(\mathcal{K}_{\psi}^{h} \rho)$, where normalization constants cancel without requiring smoothness.  

The closed-form RWPO update \eqref{approx_score_2} thus defines the \emph{sparse-transformer layer}, in which the RWPO replaces the standard dot-product attention with an optimal-transport kernel as discussed in \eqref{general_transformer}, embedding structural priors directly into the network. Furthermore, because the RWPO update admits closed-form proximal and softmax evaluations, the computational cost per layer is comparable to standard transformer layer, ensuring scalability to large token sets.

\subsection{Learning objective and training procedure}

We now make precise the training objective used to learn the unknown drift potential $\phi$ with prior $\psi$.  The loss function we propose combines three components:
\begin{equation}
\label{objective}
\mathcal{J}(\phi)
:= \KL(\rho^* \| \rho_T) \;+\; \frac{1}{2}\int_0^T \!\int \|\nabla \tphi(x,t) \|_2^2 \,\rho(x,t)\,dx\,dt
\;+\; c\,R(T)\,,
\end{equation}
where
\begin{itemize}
  \item $\rho_T$ denotes the terminal density produced by the dynamics;
  \item $c\geq0$ is a weight; and $R(T)$ is an HJB residual regularizer (defined below) that penalizes deviations from the optimality conditions associated with the variational formulation.
\end{itemize}
For the continuous dynamics, the training objective can be formulated as
\[
\min_{\rho}\, \mathcal{J}(\phi)
\quad \text{subject to} \quad
dX_t = \nabla{\phi}_t(X_t) - \nabla{\tilde{\phi}}(X_t) - \beta^{-1}\nabla\log\rho_t(X_t)\,,
\]
where $X_t\sim \rho_t$ and $\nabla\tilde{\phi}$ is approximated by a proximal operator for non-smooth prior drift potential. 

We remark that, relative to the variational formulation \eqref{rho_T_opt}, the additional HJB regularizer vanishes at the optimal solution and thus primarily serves to accelerate convergence without modifying the minimizer. 
In the following, we explain the role of each term in \eqref{objective} and present the corresponding discrete Monte Carlo estimators used during training.

\medskip

\textbf{(I) The KL divergence and its evaluation along trajectories.} 
The KL divergence can be expressed as
\[
\KL(\rho^* \| \rho_T) 
= \int \rho^*(x)\log\frac{\rho^*(x)}{\rho_T(x)}\,dx
= \int \rho^*(x)\log\rho^*(x)\,dx - \int \rho^*(x)\log\rho_T(x)\,dx\,.
\]
Since the first term is independent of the model, minimizing the KL divergence is equivalent to maximizing the log-likelihood term 
\(\int \rho^*(x)\log\rho_T(x)\,dx\).

For the simplicity of notation, let
\[
v(x,t) = \nabla \phi(x,t) - \nabla\psi(x) - \beta^{-1}\nabla\log\rho(x,t)\,.
\]
To estimate the expectation with respect to $\rho^*$, we draw training samples at the terminal time,
\[
x_j^M \sim \rho^*, \qquad j=1,\dots,N\,,
\]
and evaluate \(\log\rho_T\) at these points. This requires recovering the full trajectory \(\{x_j^k\}_{k=0}^M\). Given \(x_j^M\) at time \(T\), we integrate the backward ODE
\[
\frac{d}{dt} x_t = -v(x_t,t)\,,\qquad x_T = x_j^M\,,
\]
from \(t=T\) down to \(t=0\), so that \(x_j^0\) is the corresponding preimage with density $\rho_0$.

For completeness, we recall the log-density evolution along forward trajectories \cite{neural_ode}.

\begin{prop}[Log-density evolution along trajectories]
\label{prop:log_density_evolution}
Let \(\rho\) satisfy the transport equation
\[
\partial_t \rho + \nabla\!\cdot(\rho v) = 0\,,
\qquad \frac{d}{dt}x_t = v(x_t,t)\,.
\]
Then along any forward trajectory \(x_t\) solving \(dx_t/dt =v(x_t,t)\), the material derivative of the log-density satisfies
\begin{equation}
\label{log_density_material}
D_t \log\rho(x_t,t) \;=\; -\nabla\!\cdot v(x_t,t)\,.
\end{equation}
In particular, if \(v\) is given by
\[
v(x,t) = \nabla \phi(x,t) - \nabla \psi(x) - \beta^{-1}\nabla\log\rho(x,t)\,,
\]
then
\[
D_t \log\rho(x_t,t)
= \nabla\!\cdot\big(-\nabla\phi(x_t,t) + \nabla\psi(x_t) + \beta^{-1}\nabla\log\rho(x_t,t)\big)\,.
\]
\end{prop}

By Proposition~\ref{prop:log_density_evolution}, for each trajectory we have the identity
\[
\log\rho_T(x_j(T)) \;=\; \log\rho_0(x_j(0)) \;-\; \int_0^T \nabla\!\cdot v(x_j(t),t)\,dt\,.
\]
Hence, the log-likelihood term can be approximated by a Riemann sum with step size \(h=T/M\),
\begin{align}
\label{KL_discrete}
\int \rho^*(x)\log\rho_T(x)\,dx
&\;\approx\; \frac{1}{N}\sum_{j=1}^N \log\rho_T(x_j^M) \notag\\
&\;=\; \frac{1}{N}\sum_{j=1}^N \Big[ \log\rho_0(x_j^0) - h\sum_{k=0}^{M-1} \big(\nabla\!\cdot v\big)(x_j^k,\,t_k)\Big]\,,
\end{align}
where \(x_j^M\sim\rho^*\), \(\{x_j^k\}_{k=0}^M\) are obtained by backward integration of the dynamics (e.g., via the splitting scheme~\eqref{def_split_particle}), and \(\rho_0\) is a known simple distribution, such as a Gaussian distribution. In practice, the divergence \(\nabla\!\cdot v\) is computed either by automatic differentiation or by closed-form formulas, depending on the parametrization of \(\phi\).

\textbf{(II) The transport regularizer.}
The second term in \eqref{objective},
\[
\frac{1}{2}\int_0^T \!\int \|\nabla\tphi(x,t)\|_2^2\,\rho(x,t)\,dx\,dt\,,
\]
penalizes the kinetic energy of the transportation vector field, which promotes smooth, energy-efficient token transports.  
For $\phi_k(x) = \phi(x,t_k)$, the time-space integral is approximated using a standard Monte Carlo Riemann-sum scheme:
\begin{equation}
\label{transport_discrete}
\frac{1}{2}\int_0^T\!\int \|\nabla\tphi\|_2^2\rho\,dx\,dt
\;\approx\; \frac{h}{2N}\sum_{k=0}^{M-1}\sum_{j=1}^N 
\big\lVert\nabla\tphi_k(x_j^k) \big\rVert_2^2\,.
\end{equation}

\textbf{(III) The HJB regularizer.}
The third term, $R(T)$, enforces consistency with the Hamilton–Jacobi–Bellman (HJB) equation \cite{ot_flow}, which arises from the optimality condition for $\rho_T$ in \eqref{rho_T_opt} (see \eqref{KKT_system_2} in Section~\ref{sec:theory_opt} for the derivation). The HJB residual at state $(x,t)$ is defined by
\[
\mathcal{R}_{\mathrm{HJB}}(x,t)
:= \partial_t \phi(x,t) + \tfrac{1}{2}\big\lVert\nabla\phi(x,t)-\nabla\psi(x)\big\rVert_2^2 + \beta^{-1}\big(\Delta\phi(x,t) - \Delta\psi(x)\big)\,,
\]
and the optimality condition of \eqref{rho_T_opt} implies $\mathcal{R}_{\mathrm{HJB}}\equiv 0$.  The HJB regularizer is then
\begin{equation}
\label{HJ_regularizer}
R(T) \;:=\; \int_0^T\!\int \big|\mathcal{R}_{\mathrm{HJB}}(x,t)\big|^2 \rho(x,t)\,dx\,dt\,,
\end{equation}
which we discretize analogously to the previous terms
\begin{equation}
\label{HJ_discrete}
R(T) \;\approx\; \frac{h}{N}\sum_{k=0}^{M-1}\sum_{j=1}^N \big|\partial_t \phi_k(x^k_j) + \tfrac{1}{2}\big\lVert\nabla\phi_k(x^k_j)-\nabla\psi(x^k_j)\big\rVert_2^2 + \beta^{-1}\big(\Delta\phi_k(x^k_j) - \Delta\psi(x^k_j)\big)\big|^2\,,
\end{equation}
where $\Delta\psi$ is evaluated by the smooth approximation of $\psi$, e.g., Huber function.
\medskip 

\noindent\textbf{Remarks.} 
\begin{itemize}
\item The spatial differential operators required for the transport and HJB regularizers are already computed when evaluating the KL divergence, so including these regularizers incurs no additional differentiation cost.
\item The objective \eqref{objective} combines three components: the KL divergence enforces likelihood matching, the transport regularizer encourages energy-efficient mapping consistent with optimal transport, and the HJB regularizer ensures dynamic consistency with the underlying variational formulation.
\end{itemize}

With discrete estimators for the loss function given in \eqref{KL_discrete}, \eqref{transport_discrete}, and \eqref{HJ_discrete}, algorithm \ref{alg:training} summarizes both the generation and training stages corresponding to the forward and backward RWPO updates.

\begin{algorithm}[H]
\caption{Sparse Transformer with RWPO Operator: Generation and Training}
\label{alg:training}
\begin{algorithmic}[1]
\Require Time stamps $\{t_k\}_{k=1}^M$, training data $\{x^T_j\}_{j=1}^N$, maximal training iterations $L_{\max}$.
\Statex \rule{\linewidth}{0.6pt}
\Statex \textbf{Generation phase (with trained drift potential $\phi$)}
\Statex \rule{\linewidth}{0.6pt}
\State Sample initial tokens $x_j^0 \sim \rho_0$.
\For{$k = 0, \dots, M-1$}
    \State Propagate tokens using RWPO splitting scheme:
    \begin{equation}
    \label{def_split_particle_algo}
    \begin{cases}
    x_j^{k+1/2} = x_j^k + h\,\nabla \phi_k(x_j^k)\,, \\
x_j^{k+1} = x_j^{k+1/2} + \tfrac{1}{2}\left[\,S_{\lambda h}(x_j^{k+1/2}) - \sum_{\ell=1}^N \operatorname{softmax}(U(x_j^{k+1/2},X^{k+1/2}))_{\ell}\,x_{\ell}^{k+1/2}\right]\,.
    \end{cases}
    \end{equation}
\EndFor

\Statex \rule{\linewidth}{0.6pt}
\Statex \textbf{Training phase}
\Statex \rule{\linewidth}{0.6pt}

\State Initialize $i \gets 1$.
\While{$i \leq L_{\max}$}
    \State Construct loss $\mathcal{L}(\phi)$ using the summation of \eqref{KL_discrete},  \eqref{transport_discrete}, \eqref{HJ_discrete}.
         \For{$k = M,\dots,1$}
            \State Update tokens via RWPO splitting:
            \begin{equation}
            \label{ODE_system_1}
            \begin{cases}
            x_j^{k-1/2} = x_j^{k} + \tfrac{1}{2}\left[\,S_{\lambda h}(x_j^{k}) - \sum_{\ell=1}^N \operatorname{softmax}(U(x_j^{k},X^{k}))_{\ell}\,x_{\ell}^{k}\right],\\
            x_j^{k-1}   = x_j^{k-1/2}
                        - h\,\nabla \phi_{k-1/2}(x_j^{k-1/2} ) \,.
            \end{cases}
            \end{equation}
            \State Using Prop.~\ref{prop:log_density_evolution}, update log-densities along trajectories:
            \begin{equation}
            \label{ODE_system_2}
            \begin{cases}
             \log \rho_{k-1/2}(x_j^{k-1/2})
                = \log \rho_{k}(x_j^{k})
                  -  \frac{1}{2}\nabla\cdot\left[\,S_{\lambda h}(x_j^{k}) - \sum_{\ell=1}^N \operatorname{softmax}(U(x_j^{k},X^{k}))_{\ell}\,x_{\ell}^{k}\right]\,,\\
            \log \rho_{k-1}(x_j^{k-1})
                = \log \rho_{k-1/2}(x_j^{k-1/2}) + h\,\nabla\cdot \nabla \phi_{k-1/2}(x_j^{k-1/2} )\,. \\ 
            \end{cases}
            \end{equation}
     \EndFor
    \State Compute gradient $\nabla_\theta \mathcal{L}(\phi)$ and update parameters with a first-order optimizer (such as Adam):
    \[
    \theta \gets \text{OptimizerStep}(\theta, \nabla_\theta \mathcal{L}).
    \]
    \State $i \gets i+1$.
\EndWhile

\end{algorithmic}
\end{algorithm}

\section{Convergence and stability of flow with sparse transformer}
\label{sec_analysis}

In this section, we provide a rigorous analysis of why the sparse transformer proposed in Algorithm~\ref{alg:training} achieves faster and more robust convergence toward sparse target distributions. Our analysis proceeds along two complementary directions. First, in Section~\ref{sec_long_time_analysis}, we study how the additional sparse prior directly accelerates long-time convergence of the dynamics, as measured by the  KL divergence. Second, in Section~\ref{sec:theory_opt}, we analyze the variational optimization problem \eqref{rho_T_opt} and reveal the stabilizing role of viscosity. The proof of results in this section can be found in the Appendix.

\subsection{Convergence with the sparse prior}
\label{sec_long_time_analysis}

We examine the effect of the sparse prior $
\psi^{\lambda}(x)=\lambda\|x\|_1$
on the evolution of
\begin{equation}\label{eq:FP}
\partial_t \rho \;=\; \nabla\!\cdot\!\big(\rho\,(-\nabla\phi+\nabla \psi^{\lambda})\big)\;+\;\beta^{-1}\Delta\rho\,.
\end{equation}
The analysis proceeds in two steps. First, we study the decay of the KL divergence to show that the sparse prior accelerates convergence to the target. Second, we investigate the evolution of the second moment to quantify distributional contraction; larger values of $\lambda$ enhance both effects.

We begin with KL decay under an exact drift potential. Suppose $\phi$ is chosen so that $\rho^*(x)\propto e^{\beta \phi(x)}$. Then the forward KL dissipation satisfies
\begin{equation}\label{eq:KL-diss}
\frac{d}{dt}\,\KL(\rho_t\|\rho^*)
= -\beta^{-1} \int \big\|\nabla \log \tfrac{\rho_t}{\rho^*}\big\|^2 \rho_t\,dx
  \;-\; \lambda \int \nabla \log \tfrac{\rho_t}{\rho^*} \cdot \operatorname{sign}(x) \,\rho_t\,dx\,,
\end{equation}
where $\operatorname{sign}(\cdot)$ denotes any measurable selection of the $L_1$ subgradient and is justified via smooth approximations. Thus the $L_1$ prior contributes the nonnegative term
\[
\mathcal D(\rho_t)\;:=\;\int \nabla \log \tfrac{\rho_t}{\rho^*} \cdot \operatorname{sign}(x) \,\rho_t\,dx\,,
\]
which accelerates convergence whenever $\mathcal D(\rho_t)\ge 0$.

In the following, we assume the following lower bound, referred to as \emph{directional consistency}:
\begin{equation}
\label{KL-Fisher}
\mathcal{D}(\rho_t)\;\ge\;\gamma\sqrt{\,I(\rho_t\|\rho^*)\,}\,,\quad \forall t\ge0\,, \quad  \text{ where }\qquad I(\rho_t\|\rho^*) := \int \|\nabla \log (\rho_t/\rho^*)\|_2^2 \rho_t\,dx\,,
\end{equation}
for some $\gamma \geq 0$ and $I(\rho_t\|\rho^*)$ is the relative Fisher information. The property \eqref{KL-Fisher} holds for several common parametric families (e.g., Laplace and Gaussian families), as illustrated below.

\begin{itemize}
    \item \emph{Laplace distributions.} For $\rho^*(x)=\tfrac{b}{2}e^{-b|x|}$ and $\rho_t(x)=\tfrac{a}{2}e^{-a|x|}$ with $a,b>0$,
    \[
    \mathcal D(\rho_t)=b-a\,,\qquad I(\rho_t\|\rho^*)=(a-b)^2\,,
    \]
    hence $\mathcal D(\rho_t)=\operatorname{sign}(b-a)\sqrt{I(\rho_t\|\rho^*)}$ and \eqref{KL-Fisher} holds with $\gamma=1$ exactly when $b\ge a$.
    \item \emph{Gaussian distributions.} For $\rho^*=\mathcal N(0,\sigma^2)$ and $\rho_t=\mathcal N(0,\tau^2)$, writing $c:=\tfrac{1}{\sigma^2}-\tfrac{1}{\tau^2}$ gives
    \[
    \mathcal D(\rho_t)=\sqrt{\tfrac{2}{\pi}}\,c\,\tau\,,\qquad
    I(\rho_t\|\rho^*)=c^2\tau^2,
    \]
    so $\mathcal D(\rho_t)=\sqrt{\tfrac{2}{\pi}}\,\operatorname{sign}(c)\sqrt{I(\rho_t\|\rho^*)}$ and \eqref{KL-Fisher} holds with $\gamma=\sqrt{2/\pi}$ whenever $\tau^2\ge\sigma^2$.
\end{itemize}

\begin{prop}\label{prop_KL_L1}
Assume $\rho^*$ satisfies a log-Sobolev inequality with constant $C_{\mathrm{LS}}>0$, and suppose \eqref{KL-Fisher} holds with some $\gamma\ge0$. Let
\[
a:=\tfrac{2}{C_{\mathrm{LS}}}\,,\qquad y(t):=\KL(\rho_t\|\rho^*)\,.
\]
Then, for all $t\ge0$ and denoting $(x)_{+} = \max\{x,0\}$, we have
\begin{equation}\label{eq:KL-upper}
y(t)\;\le\;\Bigg(\Big(\sqrt{y(0)}+\tfrac{\lambda\gamma}{\sqrt{a}}\Big)e^{-\tfrac{a}{2}t} -\tfrac{\lambda\gamma}{\sqrt{a}}\Bigg)_+^{2}\,.
\end{equation}
\end{prop}

To summarize, the additional term $-\lambda\gamma/\sqrt{a}$ in \eqref{eq:KL-upper} introduces an extra linear damping effect on $\sqrt{y(t)}$, thereby accelerating the exponential decay of the KL divergence. 
Larger values of $\lambda$ correspond to stronger sparse priors, which in turn yield faster contraction of the dynamics toward the target distribution $\rho^*$. 
When $\lambda=0$, the estimate \eqref{eq:KL-upper} reduces to the classical exponential decay rate implied by the log-Sobolev inequality without prior regularization.
Below, we give more examples of when the directional consistency condition \eqref{KL-Fisher} holds.

\begin{enumerate}
\item \emph{Separable / product families.} If $\rho^*$ and $\rho_t$ are products of 1D laws and directional consistency holds in each coordinate with constants $\kappa_j$, then
\[
\mathcal D(\rho_t)\ge (\min_j \kappa_j)\sqrt{I(\rho_t\|\rho^*)} = \gamma \sqrt{I(\rho_t\|\rho^*)}\,.
\]

\item \emph{General principle.} Directional consistency occurs when the prior drift $\nabla\psi$ and the score function $\nabla\log(\rho_t/\rho^*)$ point roughly in the same correcting direction. If $\rho_t$ is more spread out than $\rho^*$, both vectors point inward, yielding $\mathcal D(\rho_t)>0\,. $ However, misalignment can lead to negative correlation, which breaks the condition \eqref{KL-Fisher}.
\end{enumerate}

These examples show that directional consistency naturally holds when the target distribution is concentrated. Increasing $\lambda$ strengthens the linear drift potential in the ODE for $\sqrt{y(t)}$, which accelerates convergence.

Beyond convergence in KL divergence, next, we show that the same $L_1$ prior introduces additional mass contraction, which we quantify via the second moment.
\begin{prop}\label{prop:moment-ordering}
Let $\rho_0 \in \mathcal P_2(\mathbb R^d)$ and denote  $\rho_t^{\lambda}$ as the solution to \eqref{eq:FP} with $\psi(x) = \lambda \|x\|_1$.
\begin{enumerate}[leftmargin=1.8em,labelsep=0.5em]
  \item[\textnormal{(i)}] For any fixed drift potential $\phi$, 
  \begin{equation}\label{eq:moment-evolution}
    \frac{d}{dt} \int \|x\|_2^2 \rho_t^{\lambda}(x)\,dx 
    = -2 \lambda \int \|x\|_1 \rho_t^{\lambda}(x)\,dx
      +2 \int x \cdot \nabla \phi(x)\, \rho_t^{\lambda}(x)\,dx
      + 2 d \beta^{-1}\,.
  \end{equation}
  In particular, the first term implies that a larger $\lambda$ accelerates the decay of the second moment.

  \item[\textnormal{(ii)}] If $\lambda_1 \ge \lambda_2 \ge 0$ and $\phi(x) = -\tfrac{1}{2}x^{\top}A x$ with $A \succeq 0$, then
  \begin{equation}\label{eq:moment-ordering}
    \int \|x\|_2^2 \rho_t^{\lambda_1}(x)\,dx 
    \;\le\;
    \int \|x\|_2^2 \rho_t^{\lambda_2}(x)\,dx\,, 
    \qquad \forall\, t \ge 0\,.
  \end{equation}
\end{enumerate}
\end{prop}

\noindent\textbf{Remarks.} 
The $L_1$ term serves as an additional damping strength in the dynamics. 
The first term on the right-hand side of \eqref{eq:moment-evolution},
\[
-2\lambda \int \|x\|_1 \rho_t^{\lambda}(x)\,dx\,,
\]
acts as a dissipative force that contracts trajectories toward the origin, thereby reducing the spread of $\rho_t^{\lambda}$ and accelerating the decay of its second moment. 
Larger values of $\lambda$ amplify this contraction. 
Together with the KL dissipation analysis, this demonstrates that the sparse $L_1$ prior not only accelerates convergence in a distributional sense but also enforces a tighter concentration of mass.

\subsection{Stability of the optimization problem}
\label{sec:theory_opt}

We now examine the variational well-posedness of the learning problem \eqref{rho_T_opt} and its associated token dynamics \eqref{FPK_particle}. 
In particular, we show that the viscosity term arising from the score function prevents gradient blow-up in the backward HJB equation, ensuring the stability of both the dynamics and the variational formulation.

To reformulate \eqref{rho_T_opt} as an unconstrained problem, we introduce a Lagrange multiplier \(\Phi\). The constrained formulation then becomes
\begin{align}
\label{eqn_uncon_1}
\argmin_{\rho_T}\;\min_{\Phi,\tphi,\rho}\;
& \KL(\rho^* \| \rho_T) 
+ \frac{1}{2}\int_0^T\!\int \|\nabla \tphi\|_2^2 \,\rho \, dx\,dt \\
&\quad + \int_0^T\!\int \Phi \Bigl[\partial_t \rho + \nabla \cdot (\rho \nabla \tphi)  - \beta^{-1}\Delta \rho \Bigr] dx\,dt \,.\notag
\end{align}

To connect the PDE dynamics with the training objective, we derive the first-order optimality conditions of \eqref{eqn_uncon_1}, which yield a coupled system
\begin{equation}
\label{KKT_system}
\begin{cases}
 \tfrac{1}{2}\|\nabla \tphi\|_2^2 - \partial_t \Phi - \nabla \Phi \cdot \nabla \tphi  - \beta^{-1}\Delta \Phi = 0\,, \\
 \partial_t \rho + \nabla \cdot (\rho \nabla \tphi)  - \beta^{-1}\Delta \rho = 0\,, \\
 \rho(\nabla \tphi - \nabla \Phi) = 0\,, \\
 \Phi_T = \tfrac{\rho^*}{\rho_T}\,,
\end{cases}
\end{equation}
where the derivation can be found in the appendix.
The third condition enforces \(\nabla \tphi = \nabla \Phi\), which reduces the system to
\begin{equation}
\label{KKT_system_2}
\begin{cases}
 \partial_t \tphi + \tfrac{1}{2}\|\nabla \tphi\|_2^2 + \beta^{-1}\Delta \tphi = 0\,, \\
 \partial_t \rho + \nabla \cdot (\rho \nabla \tphi)   - \beta^{-1}\Delta \rho = 0\,, \\
 \tphi_T = \tfrac{\rho^*}{\rho_T}\,.
\end{cases}
\end{equation}

\medskip
\noindent\textbf{Remark (Viscosity regularization).}
The coupled system \eqref{KKT_system_2} consists of a backward HJB equation for $\tphi$ with terminal condition $\tphi_T$, and a forward Fokker–Planck equation for $\rho$ with initial condition $\rho_0$. 
The Laplacian (score) term in the HJB equation acts as a viscosity term that prevents gradient blow-up. 
In the inviscid case $\beta^{-1}=0$, the equation 
\[
\partial_t \tphi + \tfrac{1}{2}\|\nabla\tphi\|_2^2 = 0
\] 
generically develops gradient singularities in finite backward time, so the KKT system \eqref{KKT_system} is well-defined only up to the first shock time. 
In contrast, when $\beta^{-1}>0$, the viscous HJB equation 
\[\partial_t \tphi + \tfrac{1}{2}\|\nabla\tphi\|_2^2 + \beta^{-1}\Delta\tphi = 0\]
admits a unique smooth solution $\tphi \in C^\infty([0,T]\times\Omega)$ for all finite $T$, provided the terminal data $u_T=\exp(\tfrac{\beta}{2}\tphi_T)$ are smooth and strictly positive, ensuring bounded gradients and no blow-up.

\medskip
\noindent\textbf{Example (1D explicit blow-up).}  
In one dimension, setting $v=\partial_x\Phi$ reduces the inviscid HJB equation to Burgers’ equation $\partial_t v + v\,\partial_x v = 0$.  
For quadratic terminal data $\Phi_T(x)=-\tfrac{a}{2}x^2$ with $a>0$, the solution 
\[
v(t,x) = -\frac{a\,x}{1-a\,(T-t)} \quad \text{for} \quad T-\tfrac{1}{a}<t\le T
\]
blows up as $t\downarrow T-1/a$.  
Thus, the inviscid KKT system is well-posed only up to this first shock time, while the viscous case avoids blow-up due to the smoothing effect of the heat operator.  
This highlights the stabilizing role of the score function $\beta^{-1}\nabla\log\rho_t$ in the token evolution \eqref{FPK_particle}.

\section{Numerical experiments}
\label{sec_NE}
To solve the optimization problem \eqref{rho_T_opt}, we parameterize the unknown drift potential $\phi$ using a residual neural network. Following the implementation in \cite{ot_flow}, we write
\[
\phi(z) = z^T N(z,\theta) + \frac{1}{2} z^T A^T A z + b^T z + c\,,
\]
where $z = [x,t]^T$, $N(z,\theta)$ is a residual network with parameters $\theta$ that captures the nonlinear component of the dynamics, while the remaining terms account for the linear part. This formulation allows both $\nabla \phi$ and $\nabla\cdot\nabla \phi$ to be computed explicitly. Combined with the explicit score function update through RWPO, this yields closed-form updates for the iterative steps in \eqref{ODE_system_1} and \eqref{ODE_system_2}.

In the experiments below, we primarily compare samples generated by the standard OT-flow \cite{ot_flow} with those from our proposed flow model incorporating the sparse transformer in Algorithm \ref{alg:training}. Additionally, the parameters $\lambda$ and $\beta$ are treated as trainable, allowing the model to select them implicitly during training. Details on model initialization, training parameters, and experimental settings are provided in Appendix~\ref{sec_detail}. \footnote{The code is in GitHub with the link \url{https://github.com/fq-han/RWPO-sparse-transformer}.}  

\subsection{Benchmark examples}

For this example, we generate new samples for several standard benchmark distributions, including double moon, rings,  2-spirals,  8-Gaussian distributions, and the checkerboard.
\begin{figure}[htbp]
\centering
\begin{adjustbox}{max width=\textwidth}
\begin{tabular}{ccccc}
\includegraphics[width=0.18\textwidth]{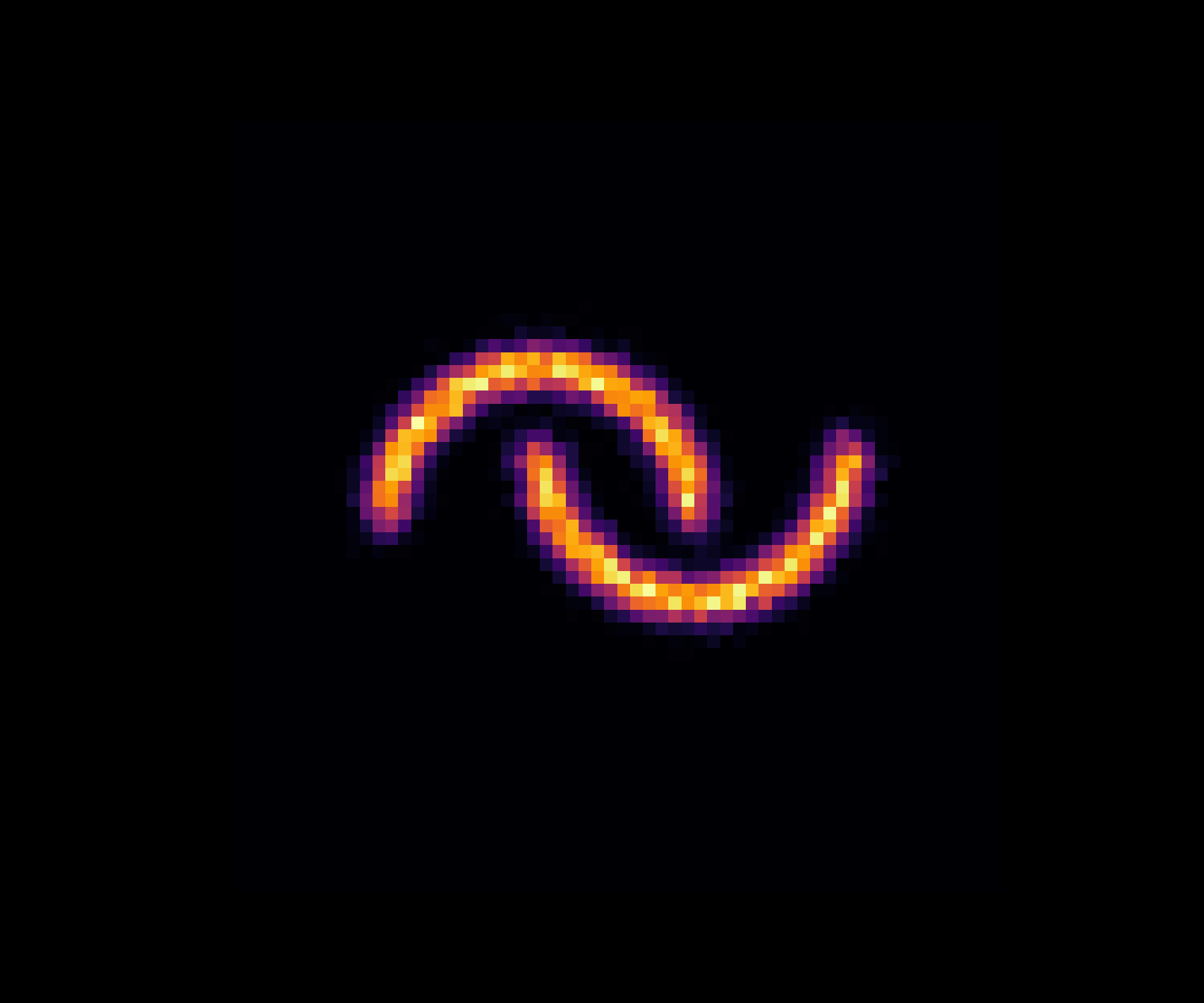} &
\includegraphics[width=0.18\textwidth]{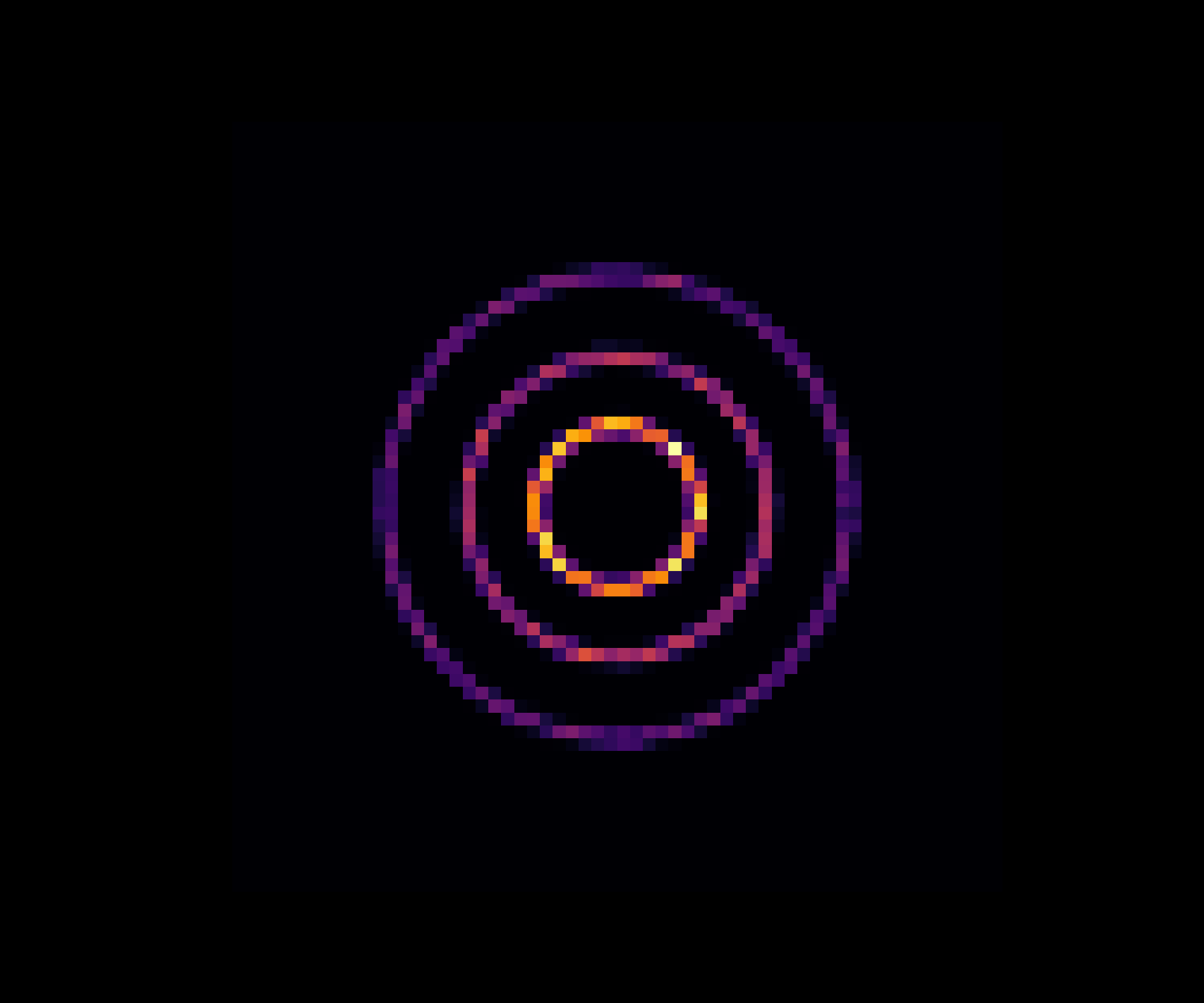} &
\includegraphics[width=0.18\textwidth]{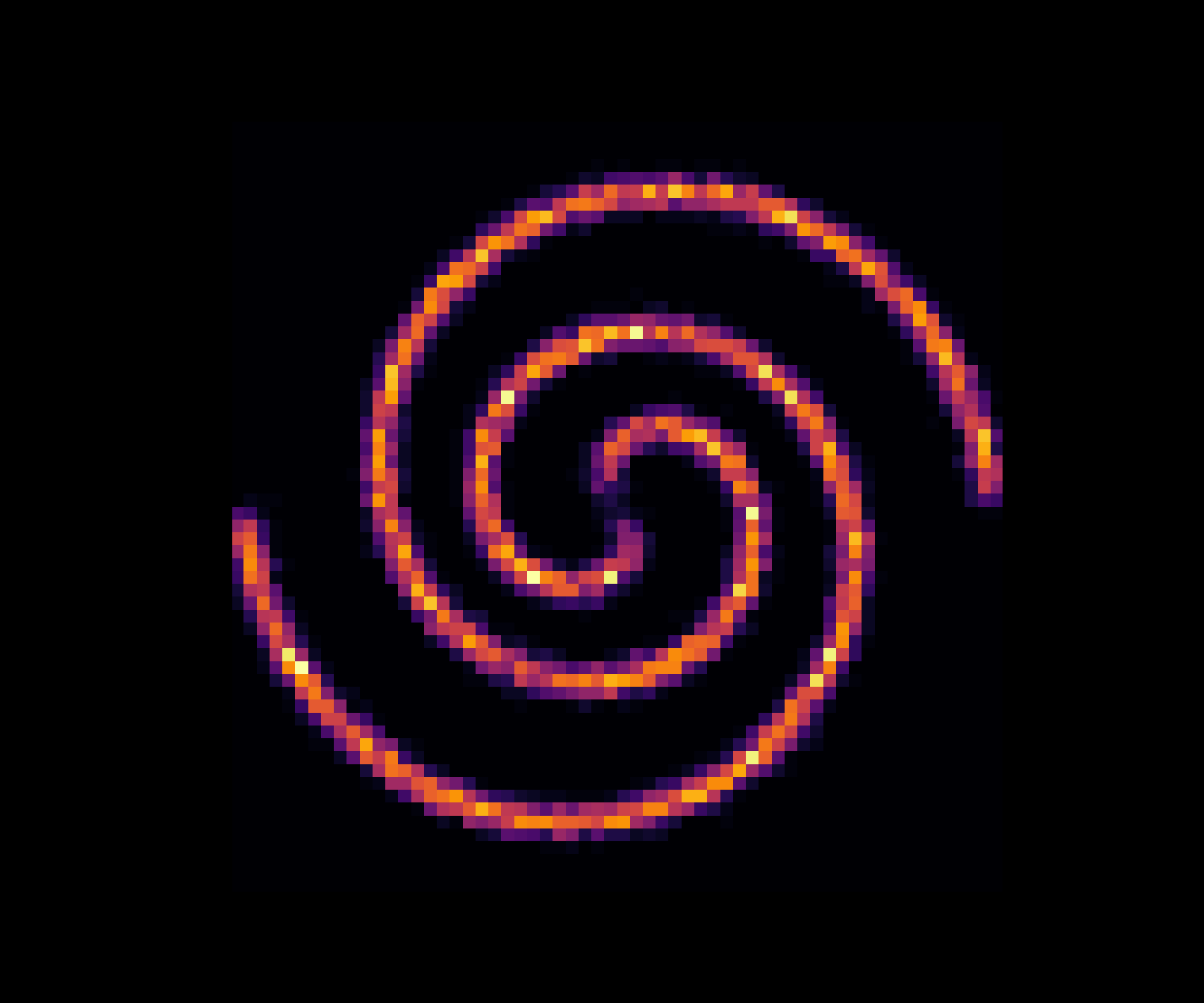} &
\includegraphics[width=0.18\textwidth]{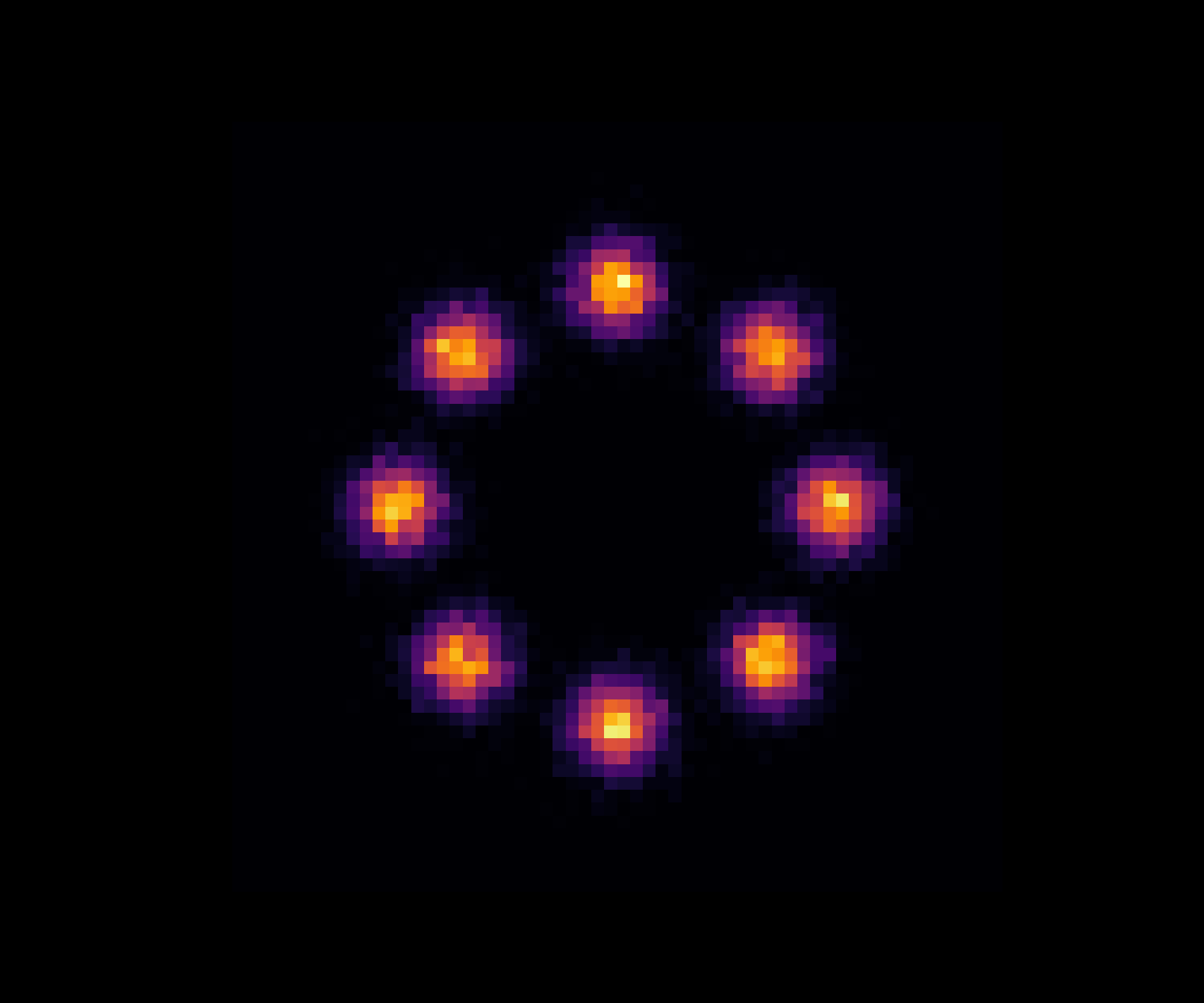} &
\includegraphics[width=0.18\textwidth]{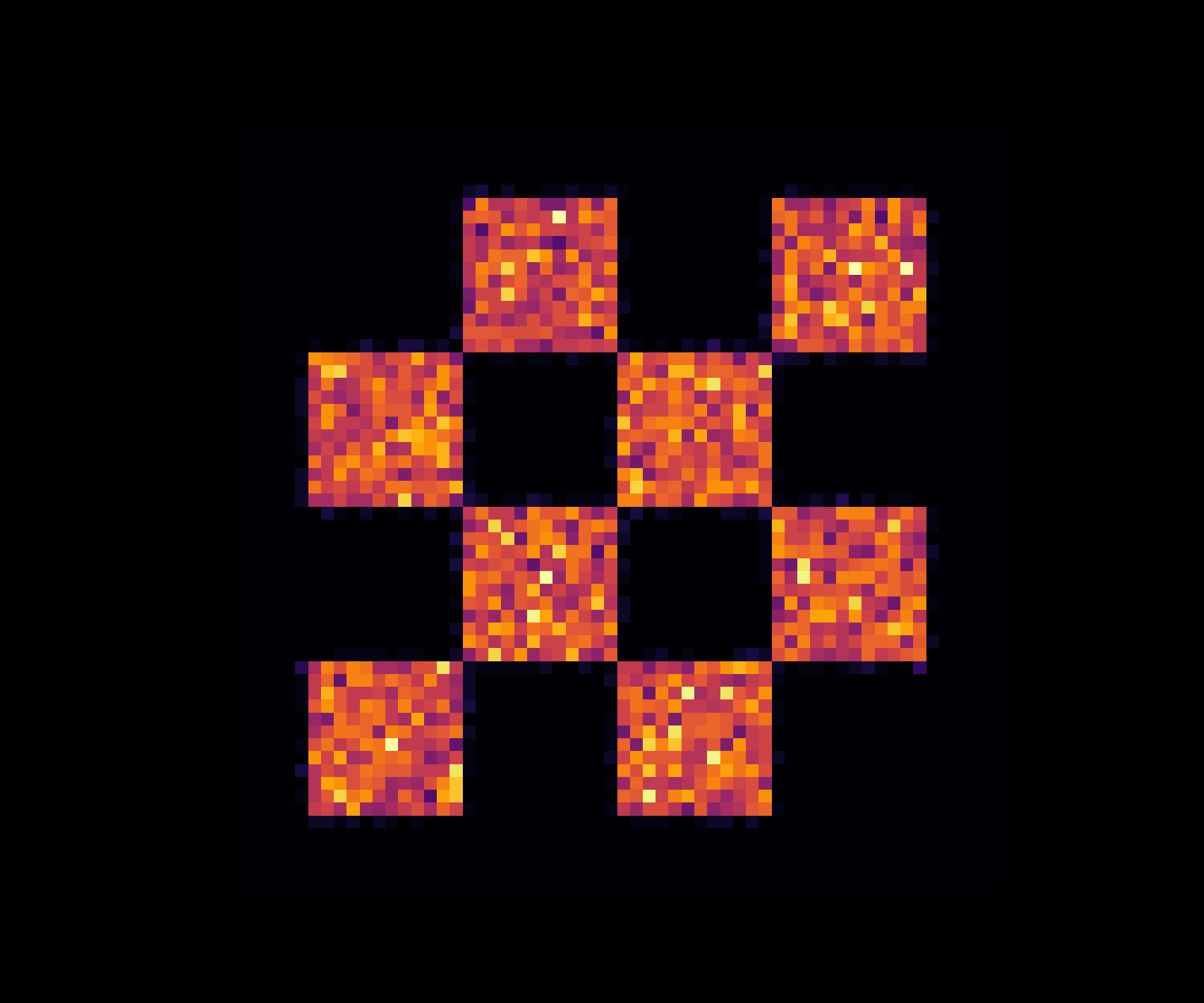} \\
\includegraphics[width=0.18\textwidth]{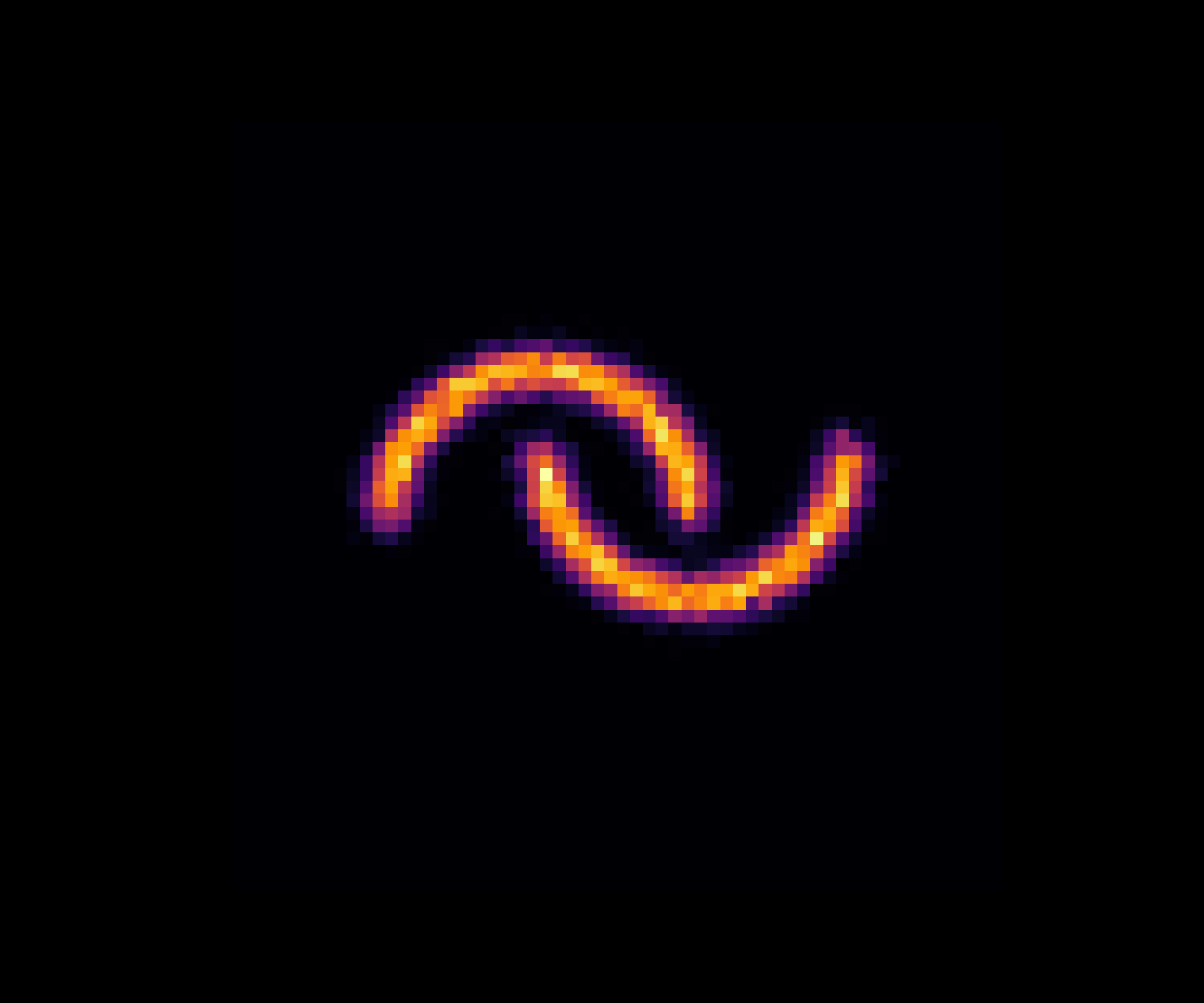} &
\includegraphics[width=0.18\textwidth]{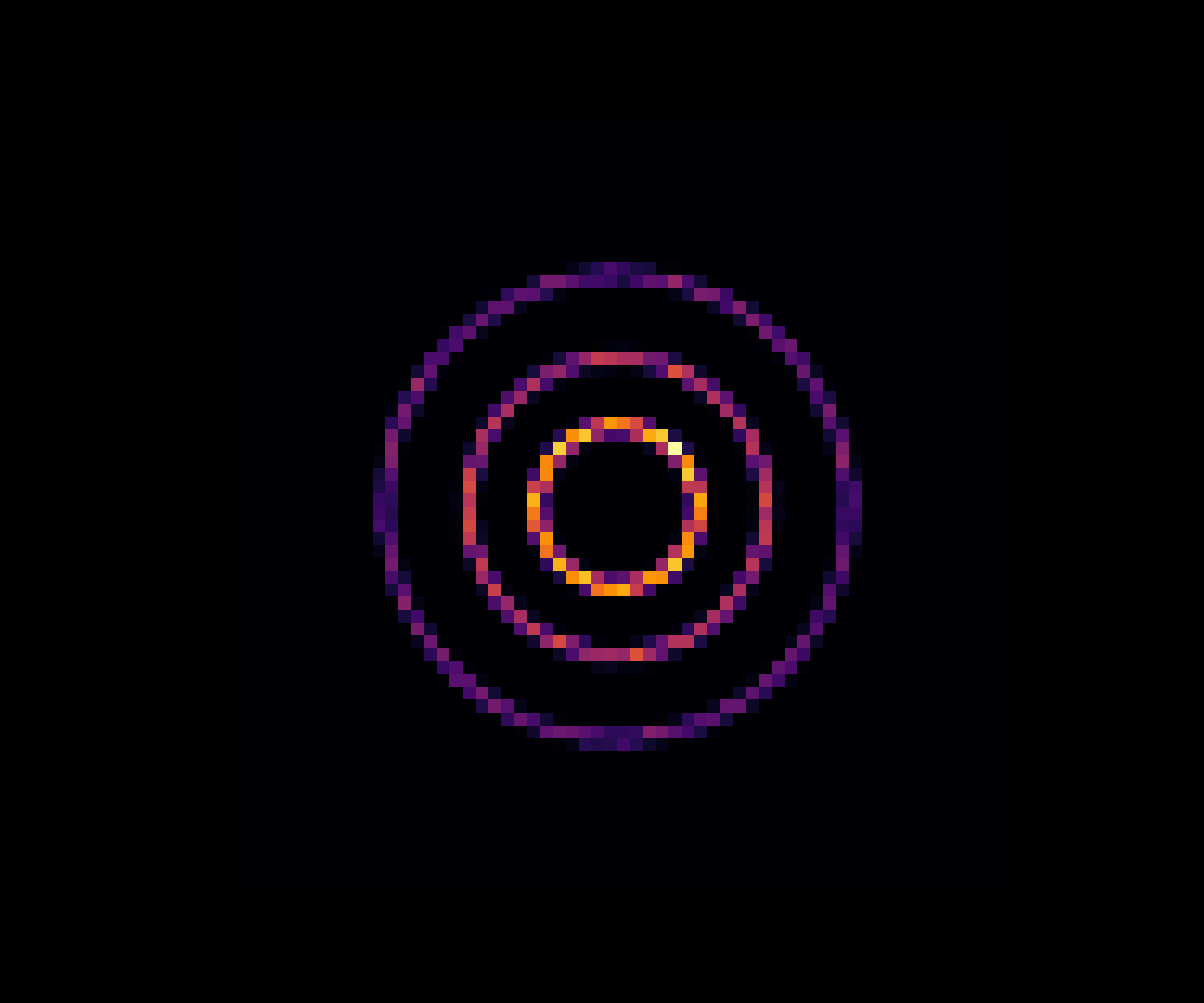} &
\includegraphics[width=0.18\textwidth]{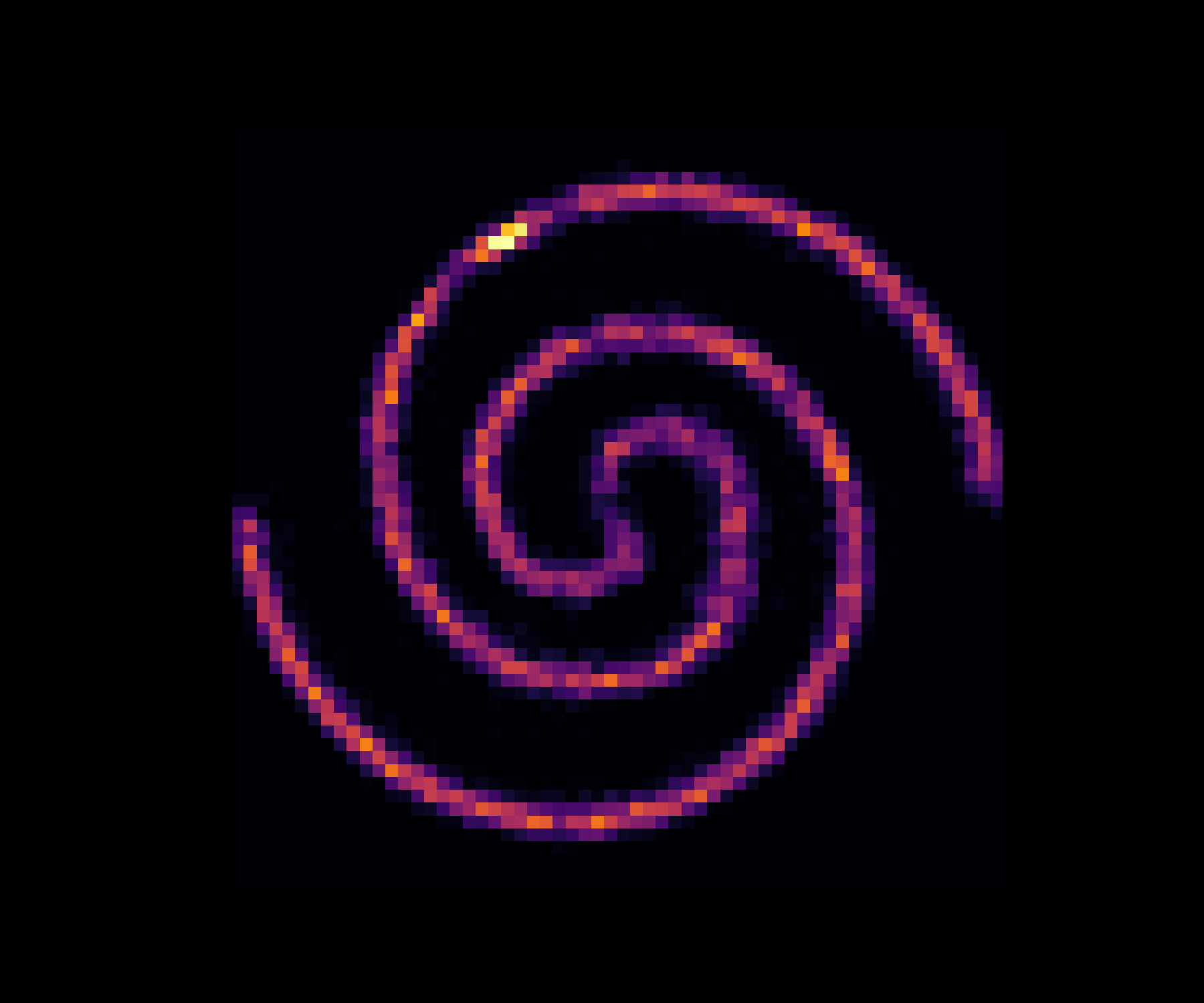} &
\includegraphics[width=0.18\textwidth]{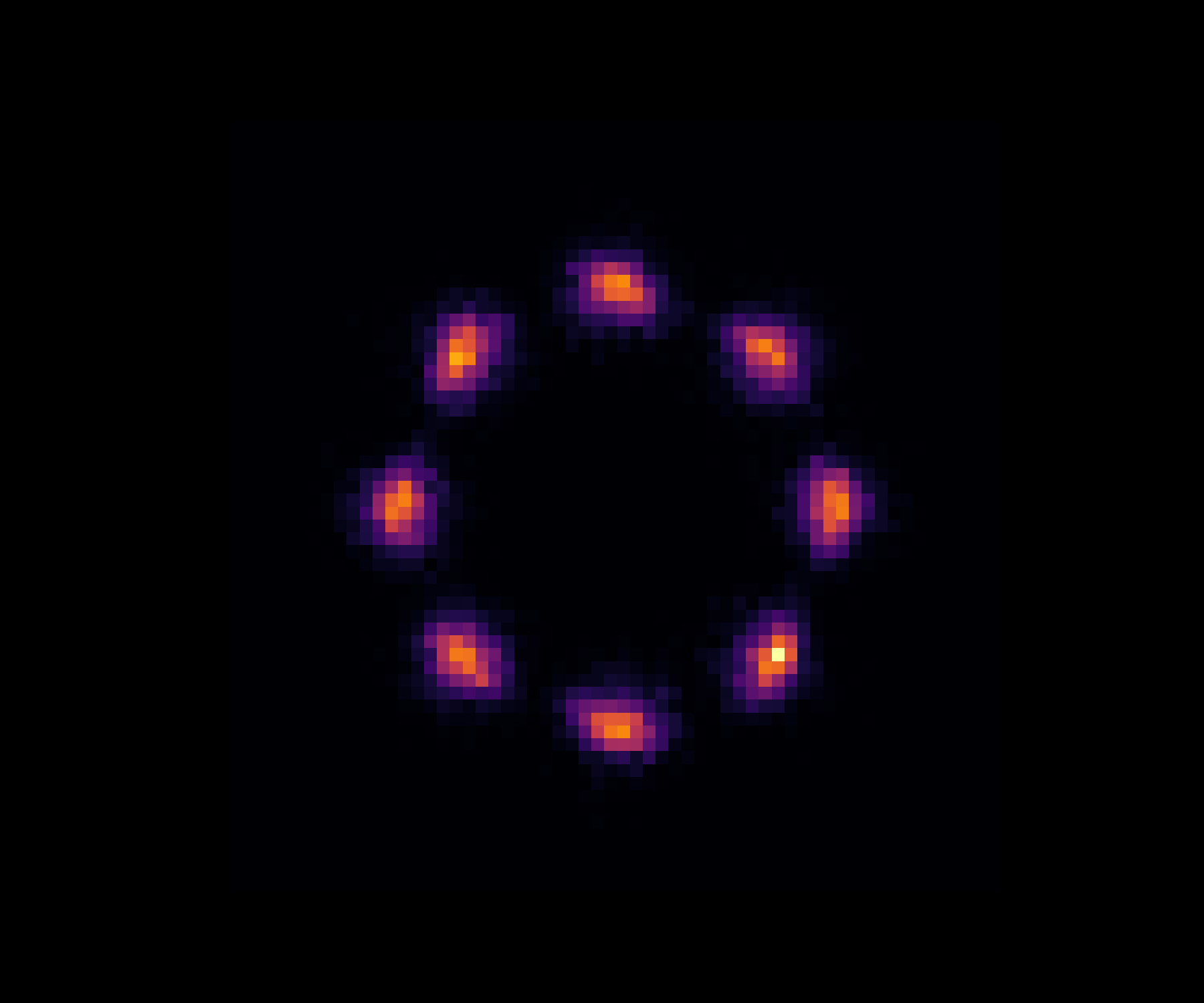} &
\includegraphics[width=0.18\textwidth]{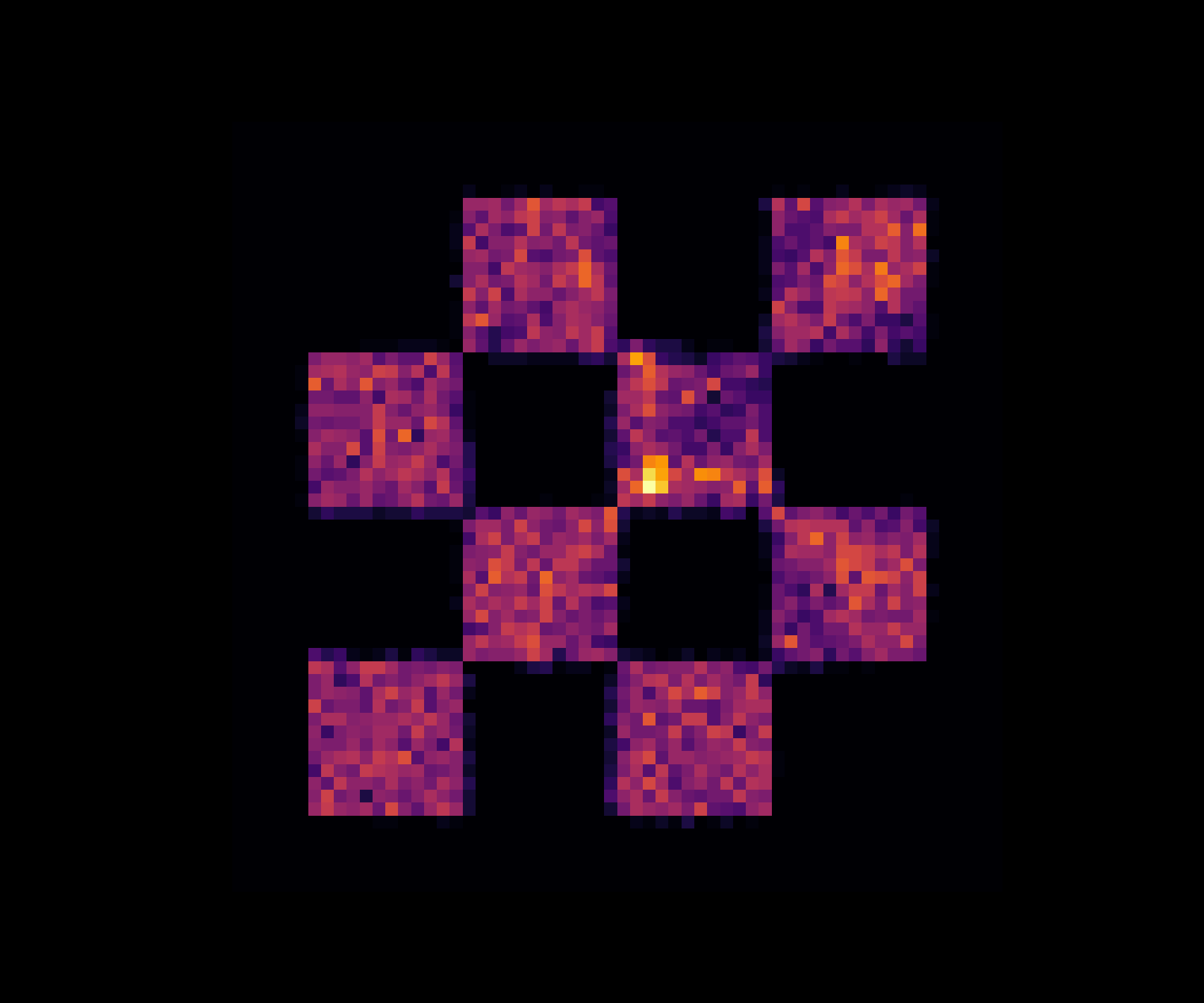} \\
\includegraphics[width=0.18\textwidth]{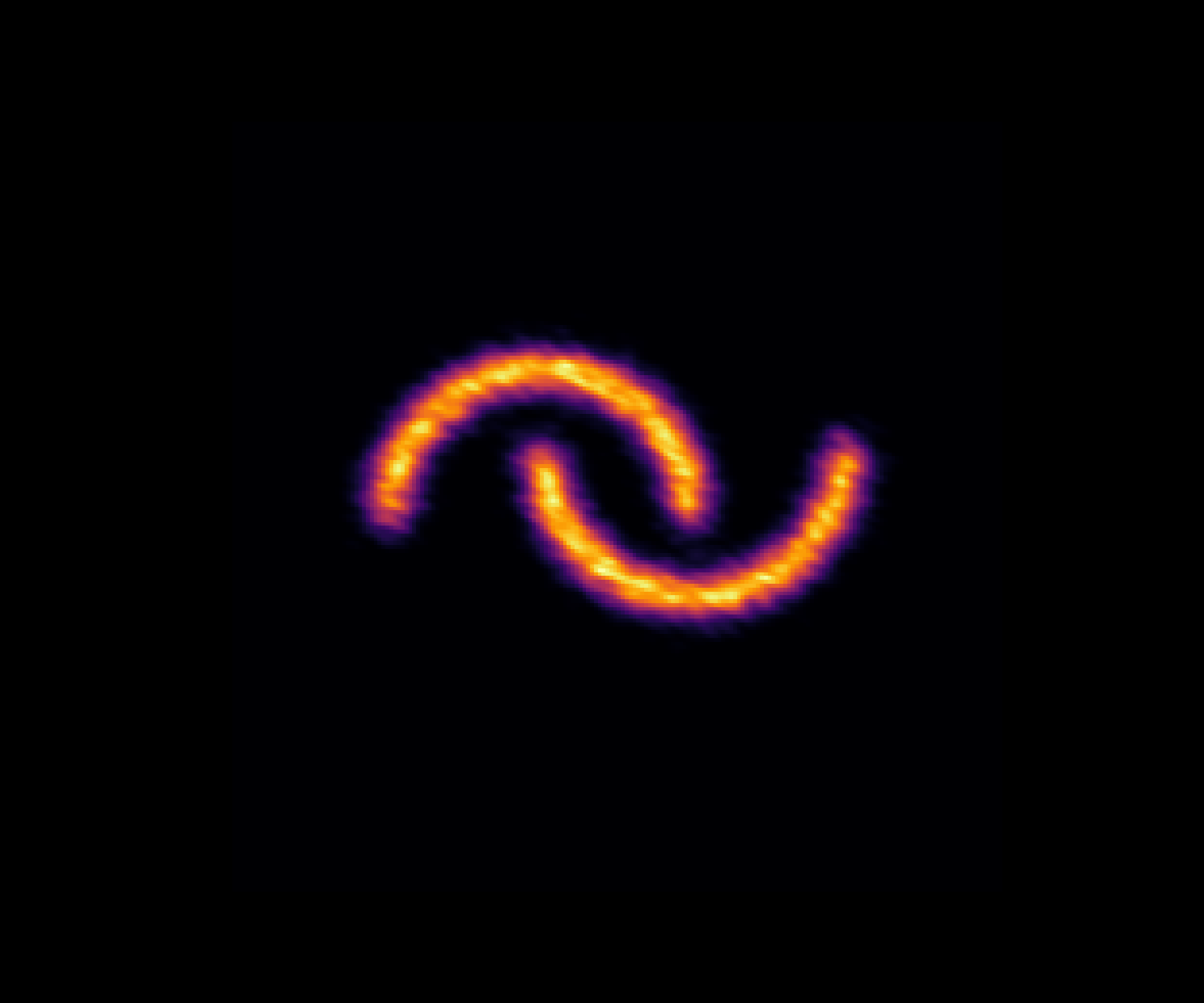} &
\includegraphics[width=0.18\textwidth]{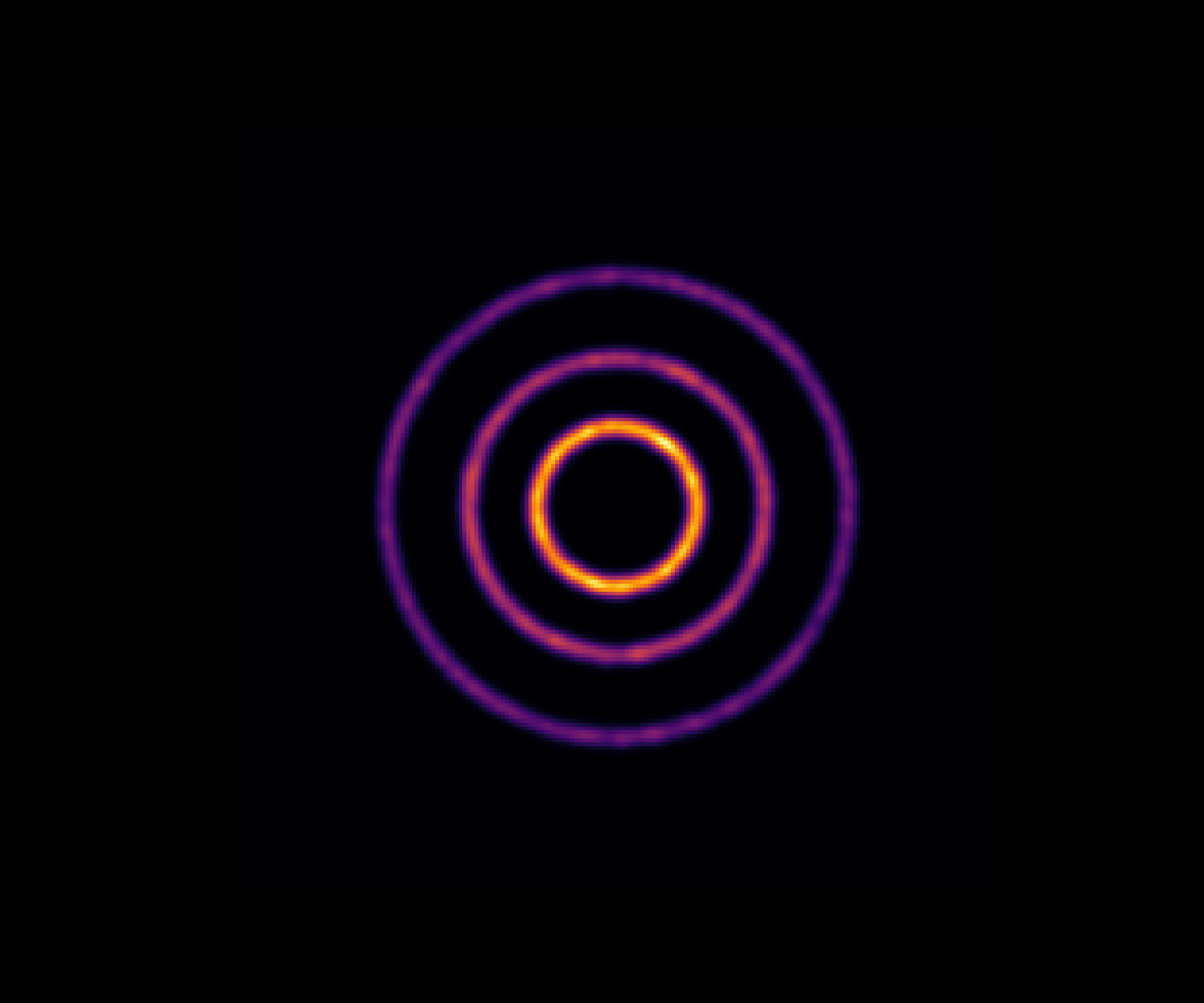} &
\includegraphics[width=0.18\textwidth]{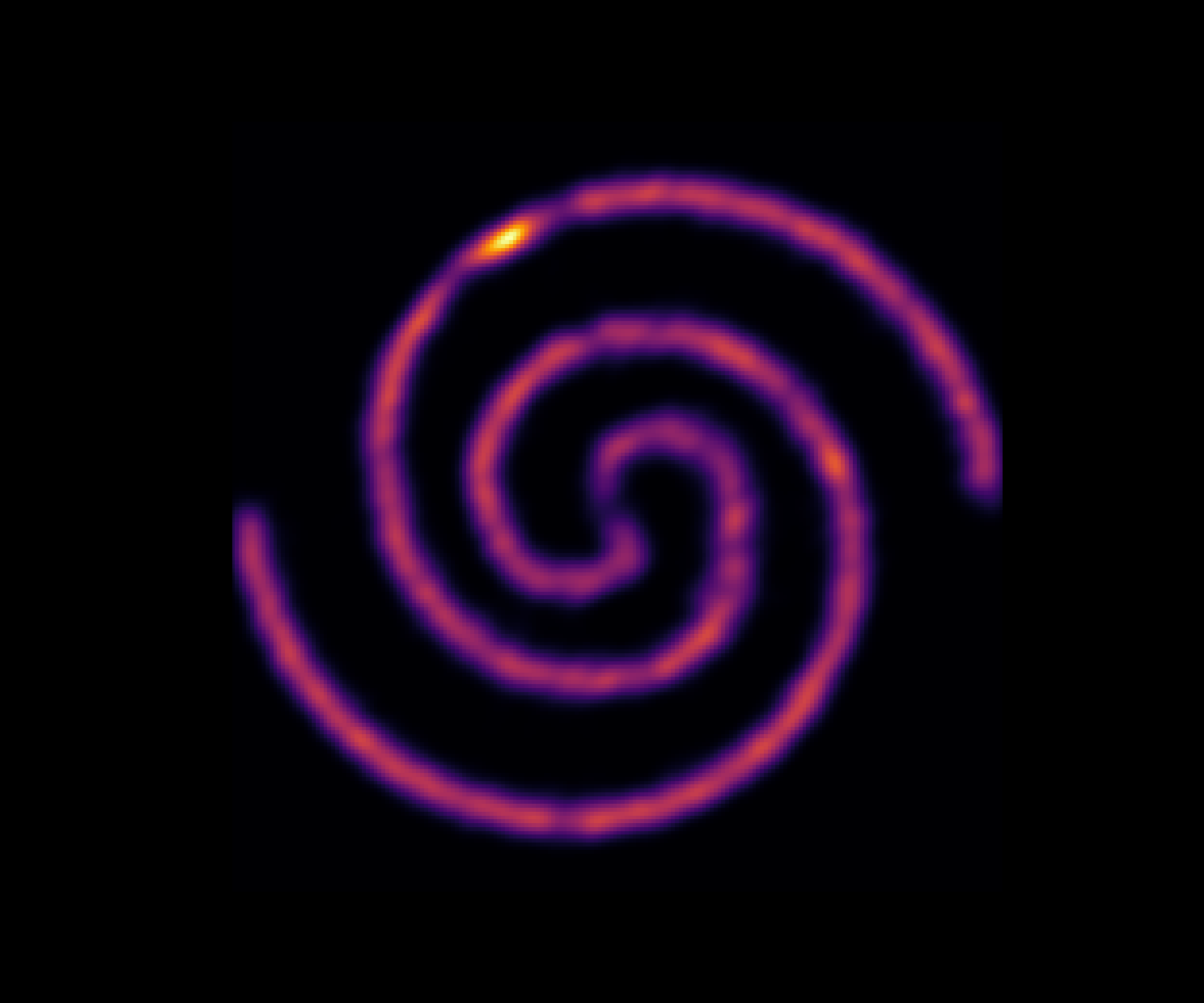} &
\includegraphics[width=0.18\textwidth]{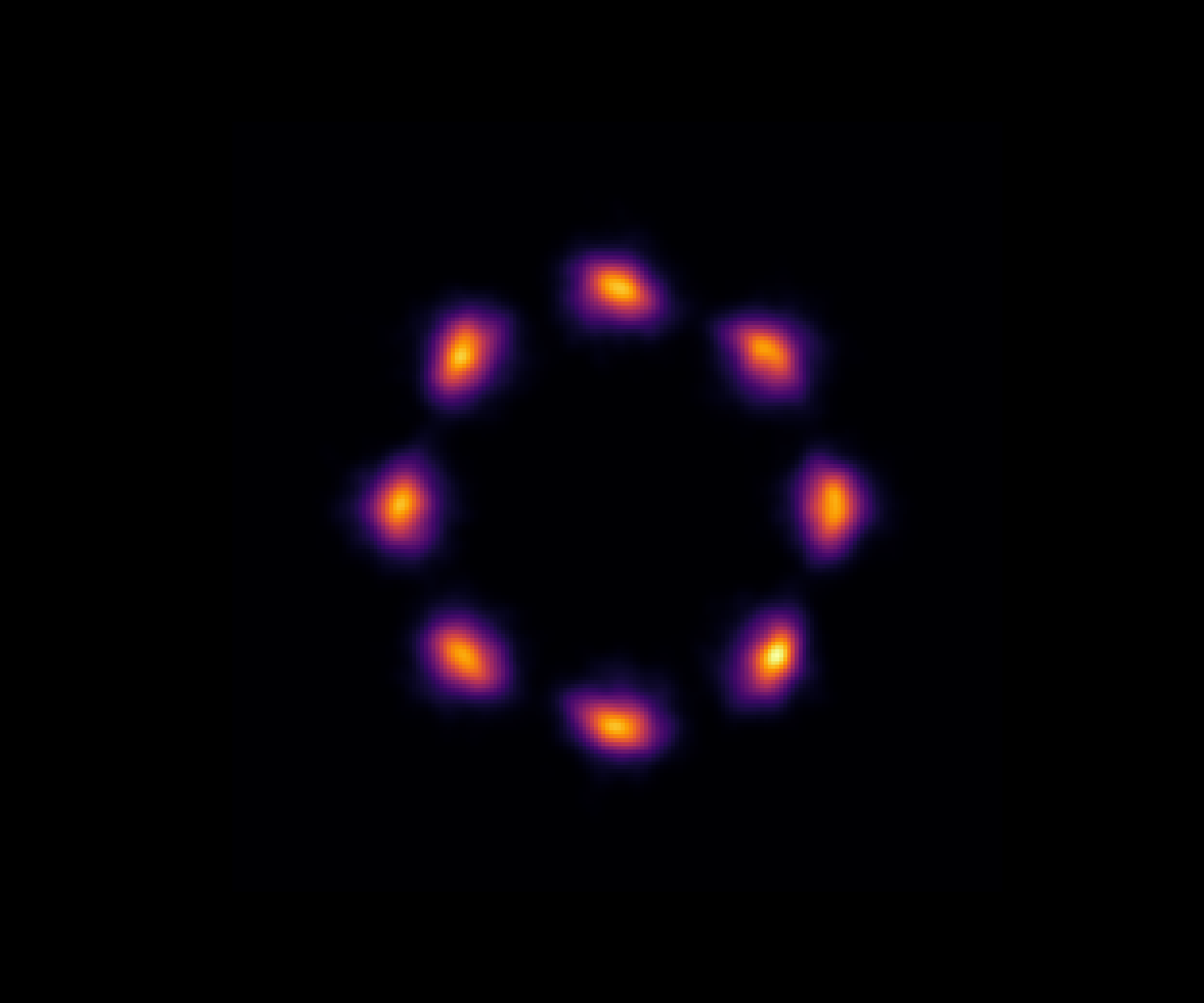} &
\includegraphics[width=0.18\textwidth]{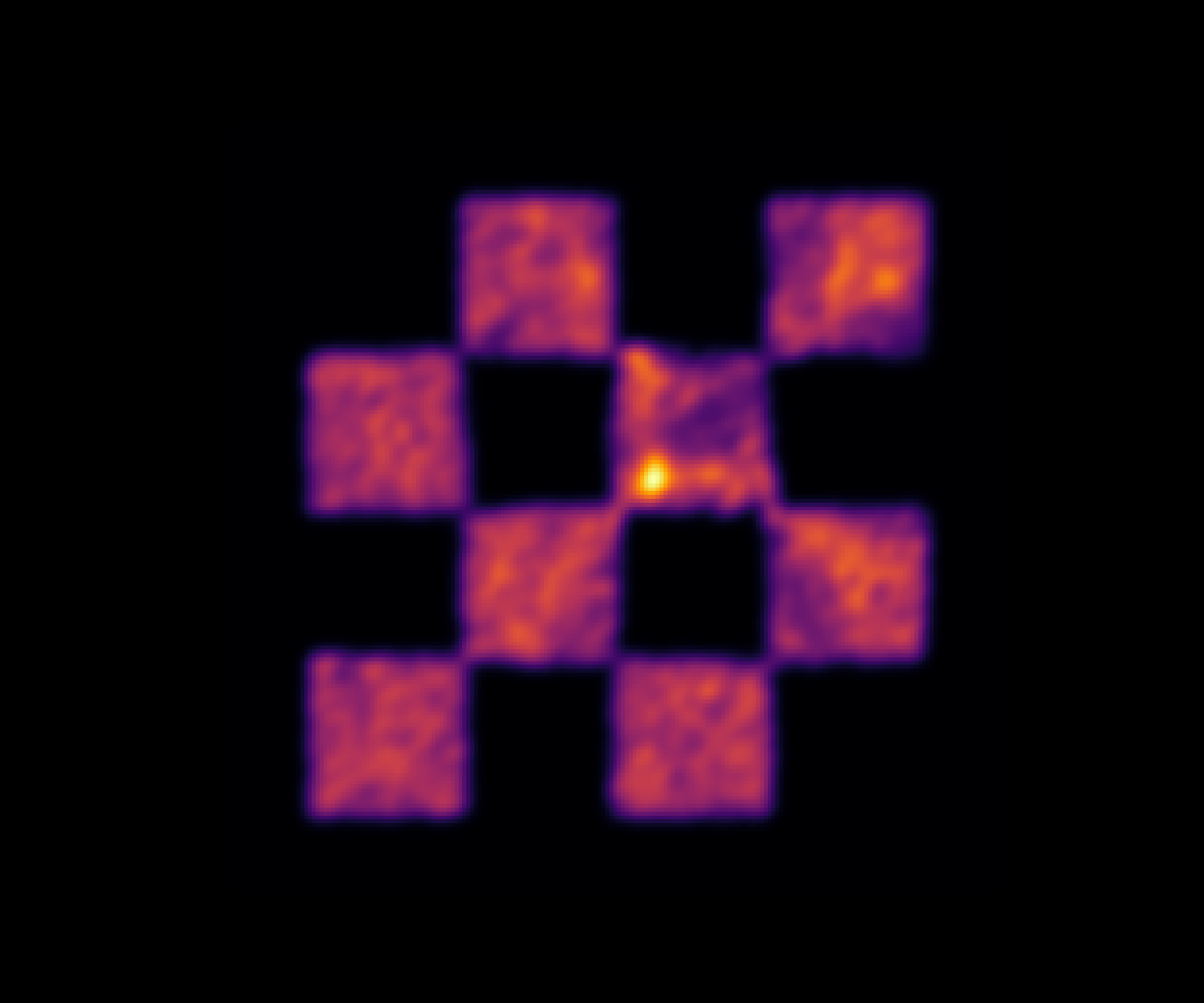} \\
Moons & Rings & 2-Spirals & 8-Gaussians & Checkerboard
\end{tabular}
\end{adjustbox}
\caption{Comparison of generated samples across benchmark datasets. Each column corresponds to a different dataset. 
Top row: target density $\rho_0$. Middle row: generated samples. Bottom row: KDE-estimated distribution using a Gaussian kernel.}
\label{fig:experiments}
\end{figure}

As shown in Fig.~\ref{fig:experiments}, the proposed flow model with the sparse transformer accurately and robustly recovers the target benchmark distributions.

\begin{figure}[htbp]
\centering
\begin{tabular}{ccc}
\includegraphics[width=0.25\textwidth]{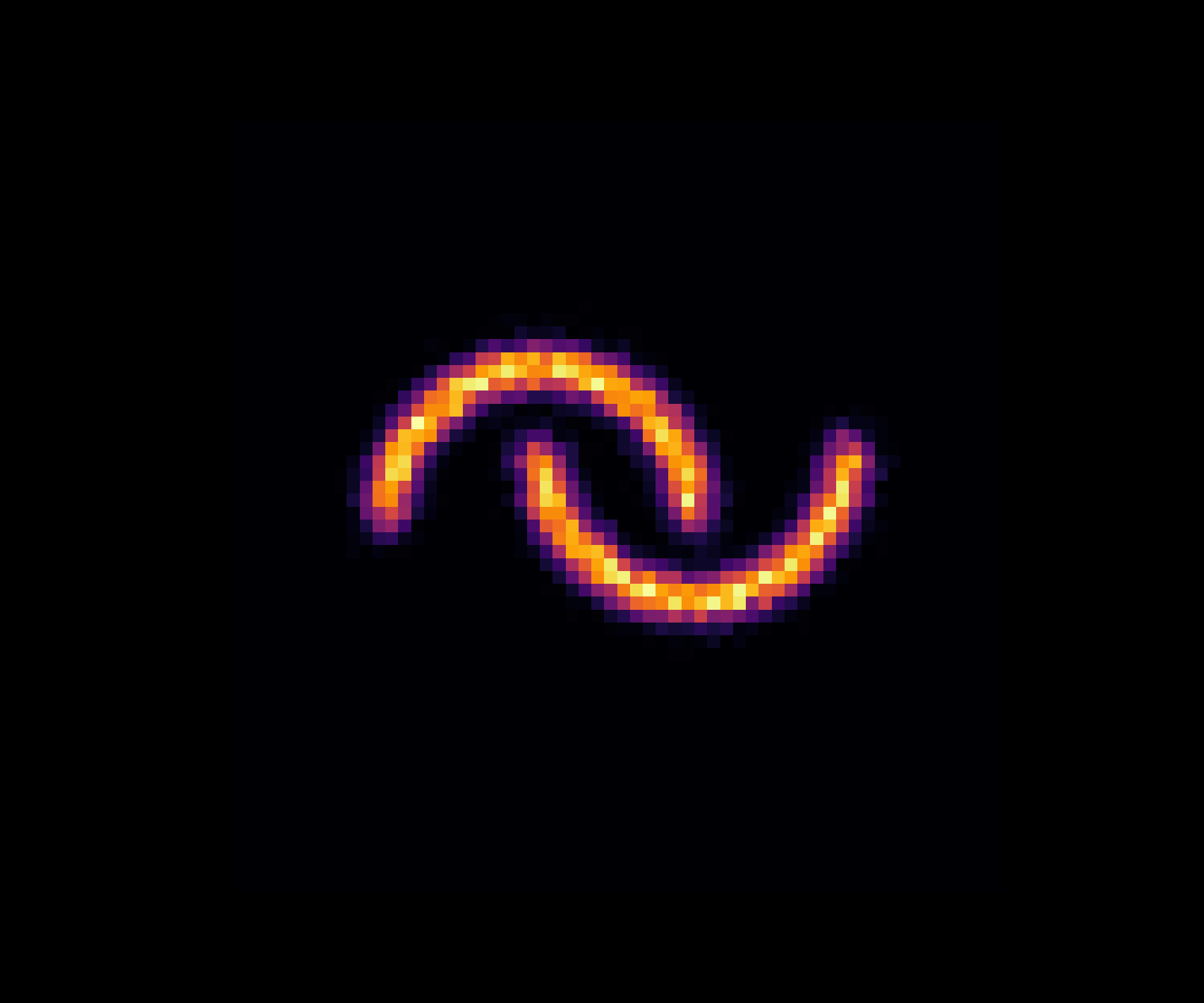} &
\includegraphics[width=0.25\textwidth]{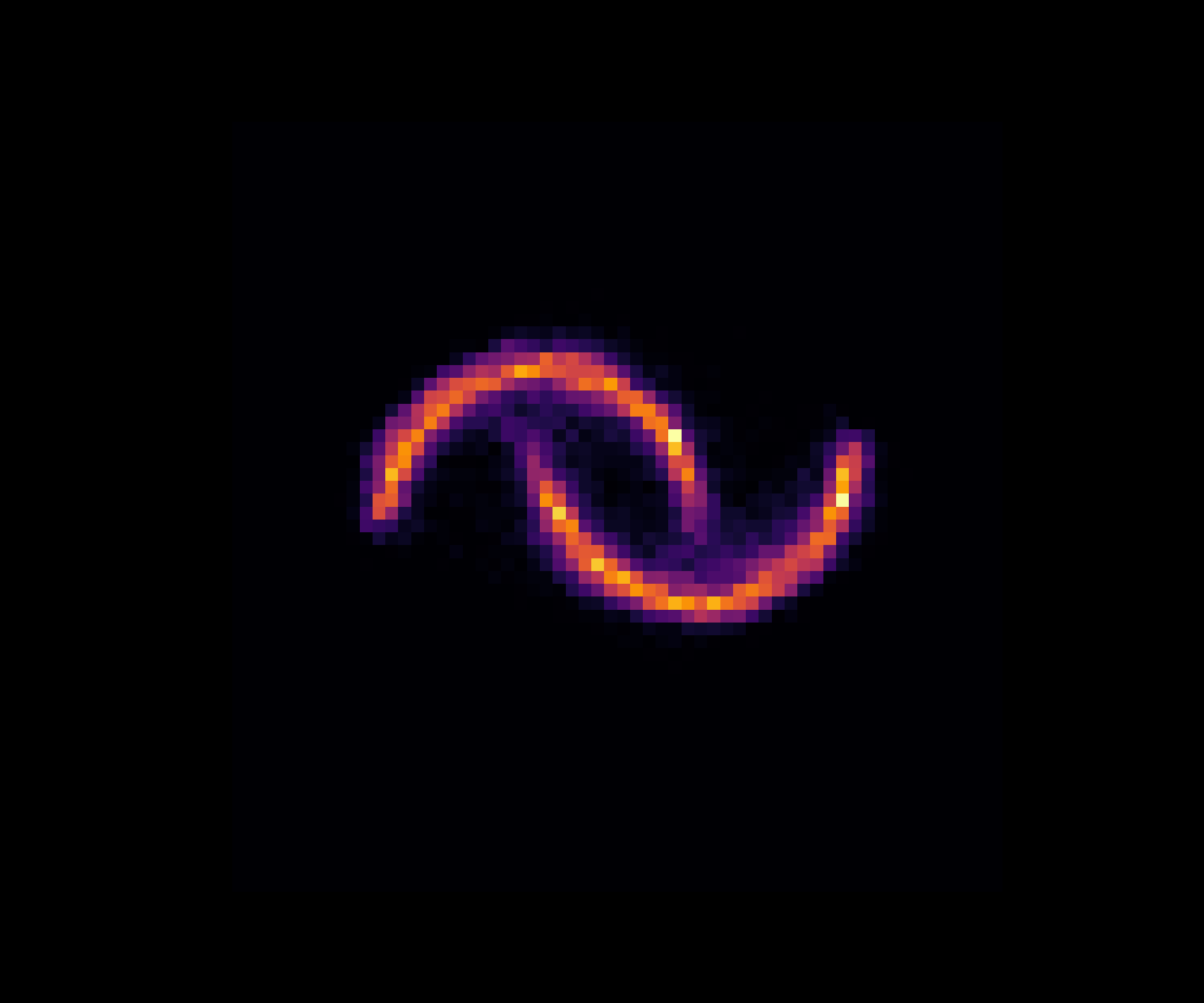} &
\includegraphics[width=0.25\textwidth]{New_figs/_moons_alph20_1_m48_lambdak2_checkpt_generated_density.png} \\
\includegraphics[width=0.25\textwidth]{New_figs/_2spirals_alph10_1_m64_lambdak2_checkpt_target_density.png} &
\includegraphics[width=0.25\textwidth]{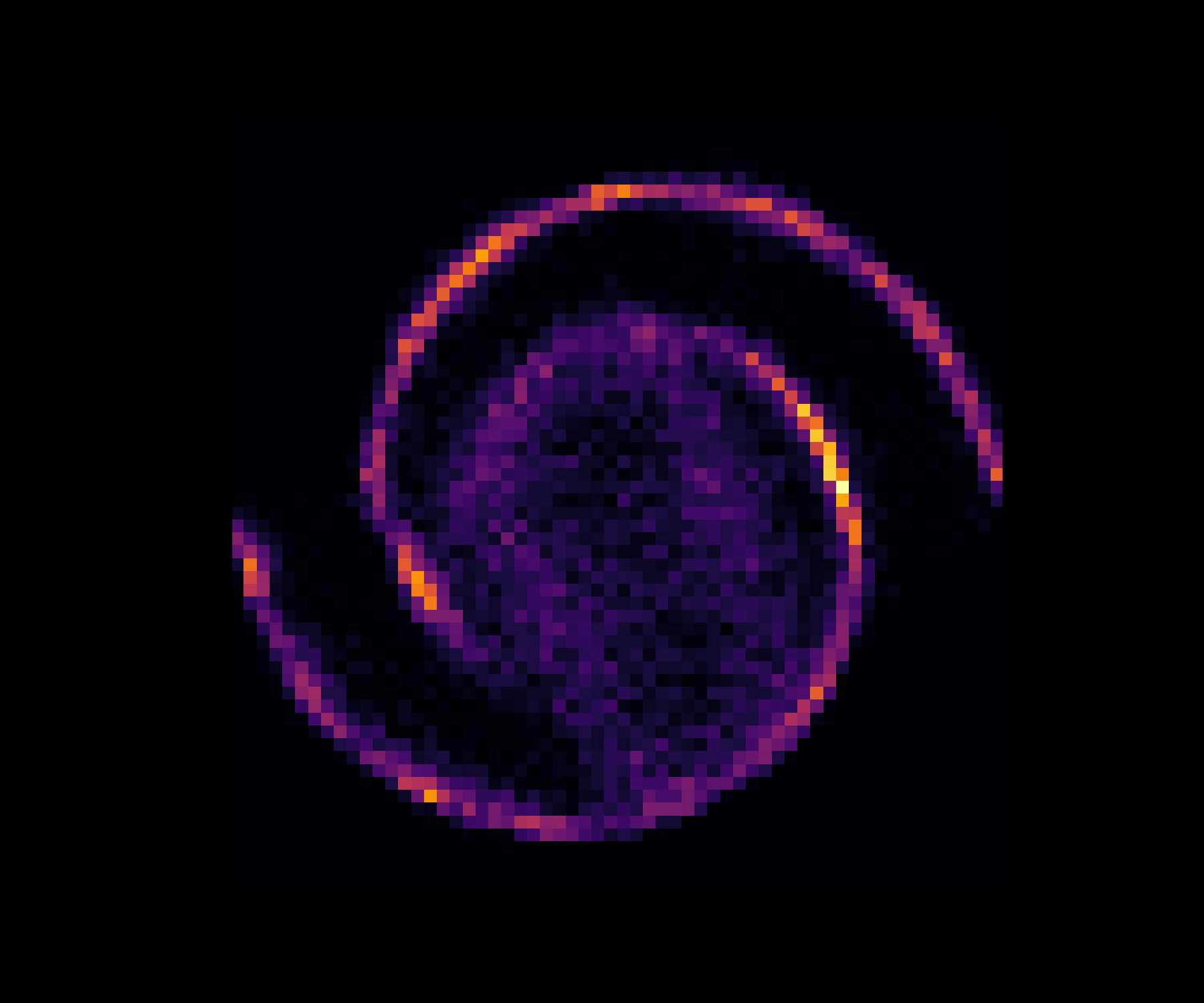} &
\includegraphics[width=0.25\textwidth]{New_figs/_2spirals_alph10_1_m64_lambdak2_checkpt_generated_density.png} \\
Target &  Generation without & Generation with\\
&sparse transformer & sparse transformer
 \end{tabular}
\caption{Top row: Moons dataset. Bottom row: 2-Spirals dataset. 
Columns show the target density and generated samples without and with the sparse transformer, respectively.}
\label{fig:moons_spirals}
\end{figure}
From Fig.~\ref{fig:moons_spirals}, we observe that the flow with the sparse transformer provides a closer approximation to the benchmark distributions, demonstrating the positive effect of the sparse prior on sample quality.

\subsection{Bayesian inverse problems}

We now apply the proposed flow model to solve Bayesian inverse problems \cite{stuart2010inverse}.  
Let $x$ denote the parameter of interest and $y$ the corresponding measurement, modeled as
\[
y = F(x) + \epsilon\,, 
\qquad \epsilon \sim \mathcal{N}(0,\sigma^2)\,,
\]
where $F$ is a known forward operator. The objective is to approximate the posterior distribution
\[
p(x \mid y) \;\propto\; p(y \mid x)\,p(x)\,,
\]
with likelihood function $p(y \mid x)$ and prior $p(x)$.  

We employ a conditional flow to approximate the posterior.  
Given training pairs $\{(x_i,y_i)\}_{i=1}^N$, let $\rho_T(x \mid y)$ denote the learned conditional density.  
The training objective is to minimize the KL divergence 
\[
\KL\!\big(\rho_T(x \mid y)\,\|\,p(x \mid y)\big) 
= \int \rho_T(x \mid y)\,\log \frac{\rho_T(x \mid y)}{p(x \mid y)} \, dx\,.
\]
Using Bayes’ rule and omitting the constant term $\log p(y)$, we have
\begin{align*}
\KL\!\big(\rho_T(x \mid y)\,\|\,p(x \mid y)\big)
&= \int \rho_T(x \mid y)\,\log \rho_T(x \mid y)\,dx 
- \int \rho_T(x \mid y)\,\log p(y \mid x)\,dx \\
&\quad - \int \rho_T(x \mid y)\,\log p(x)\,dx\,.
\end{align*}

Approximating the integrals by Monte Carlo sampling with $x_i^T \sim \rho_T(x \mid y)$ yields the empirical loss
\[
\mathcal{L} \;\approx\; \frac{1}{N} \sum_{i=1}^N 
\Big[ 
\log \rho_T(x_i^T \mid y) 
- \log p(y \mid x_i^T) 
- \log p(x_i^T) 
\Big]\,,
\]
where $x_i^0 \sim \rho_0$ from a simple base distribution, and $x_i^T$ is obtained via the evolution \eqref{def_split_particle_implicit} with the drift potential $\phi_y(x,t)$.  
For this application, the training step in Algorithm~\ref{alg:training} corresponds to integrating forward in time from $t = 0$ to $T$, so the system in \eqref{ODE_system_1}-\eqref{ODE_system_2} is replaced by its forward-time counterpart.  

In the experiments below, we consider a Gaussian likelihood function and a Laplace prior:
\[
p(y \mid x) \;\propto\; \exp\!\left(-\frac{1}{2\sigma^2}\,\|y - F(x)\|_2^2\right),
\qquad
p(x) \;\propto\; \exp(-\lambda \|x\|_1)\,.
\]
In this case, the empirical loss simplifies to
\[
\mathcal{L} \;\approx\; \frac{1}{N} \sum_{i=1}^N 
\Big[
\log \rho_T(x_i \mid y) 
+ \frac{1}{2\sigma^2}\|y - F(x_i)\|_2^2
- \log p(x_i)
\Big] + C\,,
\]
where $C$ is independent of $x_i$.  

To accelerate training, we incorporate a supervised term $\|x_i^T - x_i\|_2^2$ into the loss function.  
This encourages the flow to map samples more accurately while retaining the theoretical properties established in Section~\ref{sec_analysis}.  
Moreover, the transport cost and the HJB regularizer in \eqref{objective} remain unchanged.

Training the conditional flow with the drift potential $\phi_y$ thus transports the base distribution $\rho_0$ to the conditional law $\rho_T(x \mid y)$, enabling efficient posterior sampling for new measurements $y$.  
In the following, we illustrate this approach on three representative inverse problems.  

\begin{enumerate}

\item \textbf{Elliptic inverse problem.}  
We first consider the problem of recovering a function $u$ from observations $w$ governed by the elliptic PDE
\[
-\alpha \Delta w + w = u \quad \text{in } \Omega\,, 
\qquad \frac{\partial w}{\partial \nu} = 0 \quad \text{on } \partial \Omega\,,
\]
where $\alpha > 0$ is a fixed parameter. The measured data are
\[
y = w|_\Gamma + \zeta = (F u)|_\Gamma + \zeta\,,
\]
with $\Gamma \subset \Omega$ denoting the observation points.  
The inverse problem consists of reconstructing $u$ from these noisy, partial observations of $w$.  

To simplify the recovery, we assume that $u$ admits a Gaussian bump basis expansion
\[
u(x) = \sum_{j=1}^d \exp\!\Big(-\frac{|x-x_j|^2}{2\sigma^2}\Big)\,,
\]
with $\sigma = 0.1$ and $x_j$ uniformly distributed in $[0,1]$.  
We then apply the proposed sparse transformer to learn the posterior distribution of $u$.

\begin{figure}[ht]
    \centering
    \begin{subfigure}[t]{0.26\textwidth}
        \centering
        \includegraphics[width=\textwidth,trim=0 0.2cm 0 0.8cm,clip]{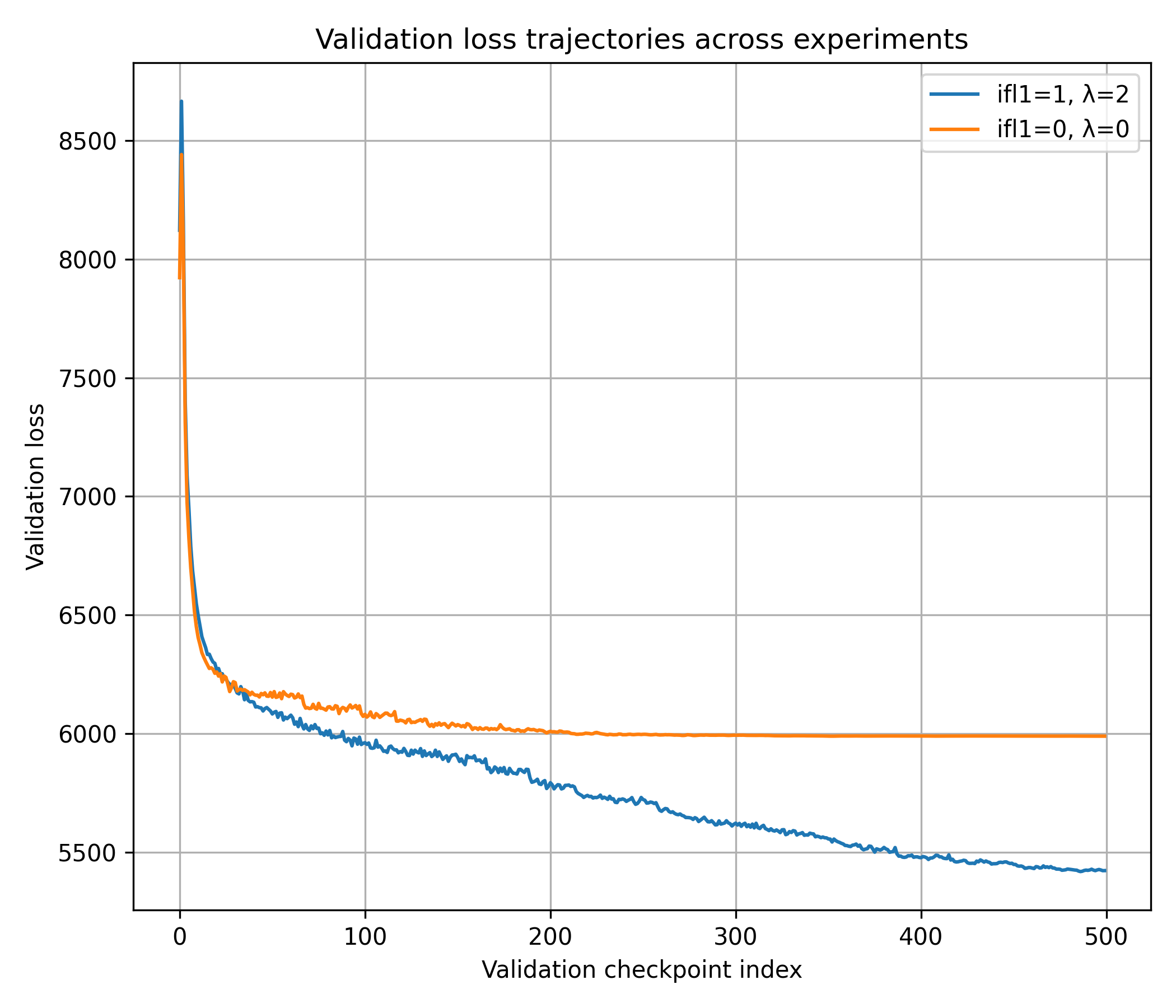}
        \caption{Validation loss trajectory.}
        \label{fig:validation_loss}
    \end{subfigure}
    \hfill
    \begin{subfigure}[t]{0.32\textwidth}
        \centering
        \includegraphics[width=\textwidth,trim=1cm 0cm 0.3cm 0.8cm,clip]{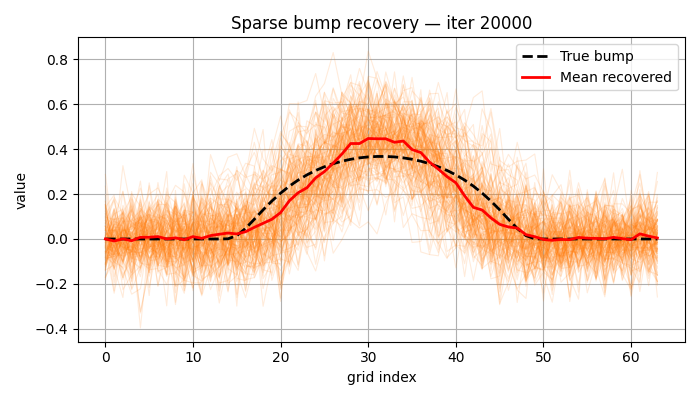}
        \caption{Recovery without sparse transformer.}
        \label{fig:bump_lambda0}
    \end{subfigure}
    \hfill
    \begin{subfigure}[t]{0.32\textwidth}
        \centering
        \includegraphics[width=\textwidth,trim=1cm 0cm 0.3cm 0.8cm,clip]{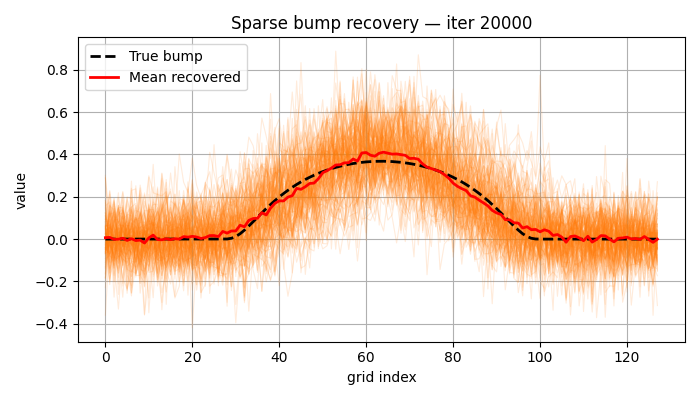}
        \caption{Recovery with sparse transformer.}
        \label{fig:bump_lambda2}
    \end{subfigure}
 
    \caption{Comparison of recovery results. The first plot shows validation loss trajectories, where blue and orange correspond to with and without the sparse transformer. The last two plots show reconstructions without and with the sparse transformer. Dotted: true coefficients; red: sample mean; orange: individual samples.  }
    \label{fig:conditioned_inverse}
\end{figure}

From Fig.\,\ref{fig:conditioned_inverse}, the first plot implies the validation loss decreases more rapidly when using the sparse transformer than without it. Moreover, we see that the flow with the sparse transformer provides a more accurate approximation for the solution of the inverse problem. The distribution of the recovered samples in the last plot is also closer to Gaussian, consistent with the underlying noise distribution, as reflected in the spread of the orange curve.

\item \textbf{Lorenz-63 system.}  
We examine our flow model on parameter inference for the Lorenz-63 chaotic dynamical system following the setup in \cite{huang2022iterated}:
\[
\begin{cases}
\frac{d}{dt}x_1 = \sigma(x_2 - x_1)\,, \\ 
\frac{d}{dt}x_2 = x_1(r - x_3) - x_2\,, \\ 
\frac{d}{dt}x_3 = x_1 x_2 - \beta x_3\,,
\end{cases}
\]
with three unknown parameters $\sigma, r, \beta$. After a spin-up time $T_{\text{spin}}$, we collect averaged observations over an interval $T_{\text{ob}}$:
\[
y_{i,j} = \int_{T+jT_{\text{ob}}}^{T+(j+1)T_{\text{ob}}} x_i(t) \, dt + \zeta_i\,, 
\qquad \zeta_i \sim \mathcal{N}(0, \sigma_{\text{noise}}^2)\,,
\]
for $i = 1, 2, 3$ and $j = 0, \dots, n_{mea}$.  
The inverse problem is to recover $(\sigma, r, \beta)$ from noisy averaged trajectory data $\{y_i\}$. 

\begin{figure}[ht]
    \centering
     \begin{subfigure}[t]{0.28\textwidth}
        \centering
        \includegraphics[width=\textwidth]{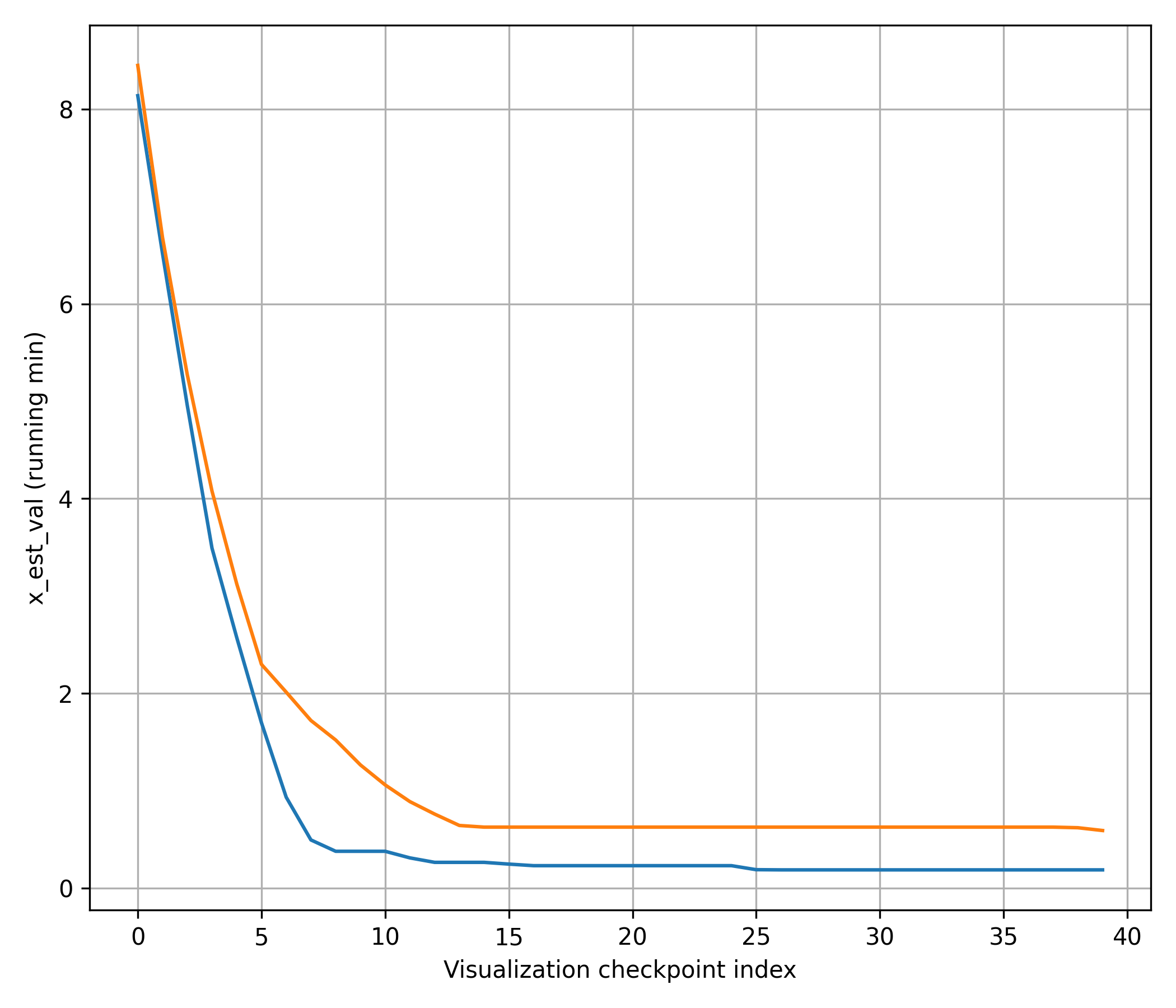}
        \caption{Validation loss trajectory.}
    \end{subfigure}
    \begin{subfigure}[t]{0.28\textwidth}
        \centering
        \includegraphics[width=\textwidth]{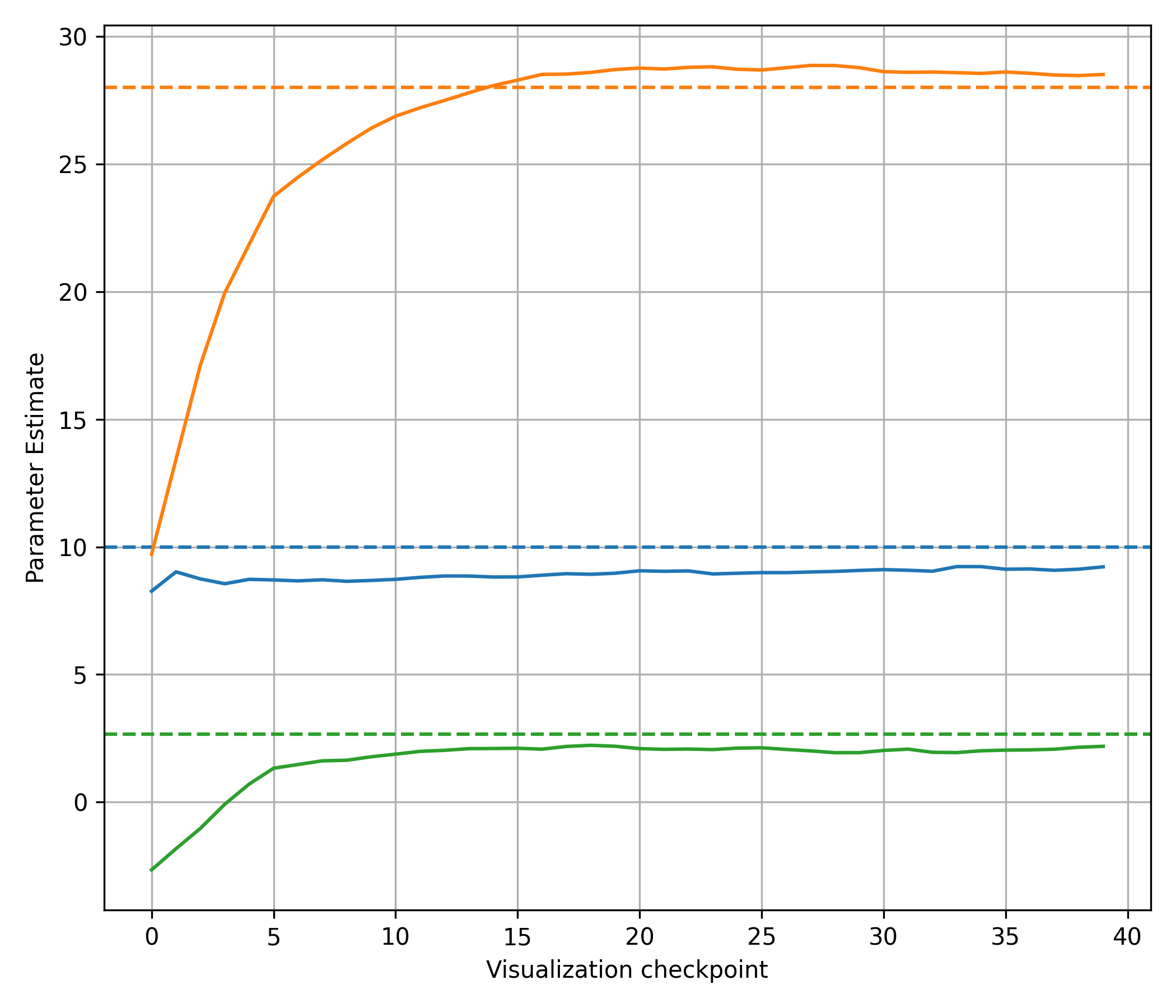}
        \caption{Without sparse transformer.}
    \end{subfigure}
    \begin{subfigure}[t]{0.28\textwidth}
        \centering
        \includegraphics[width=\textwidth]{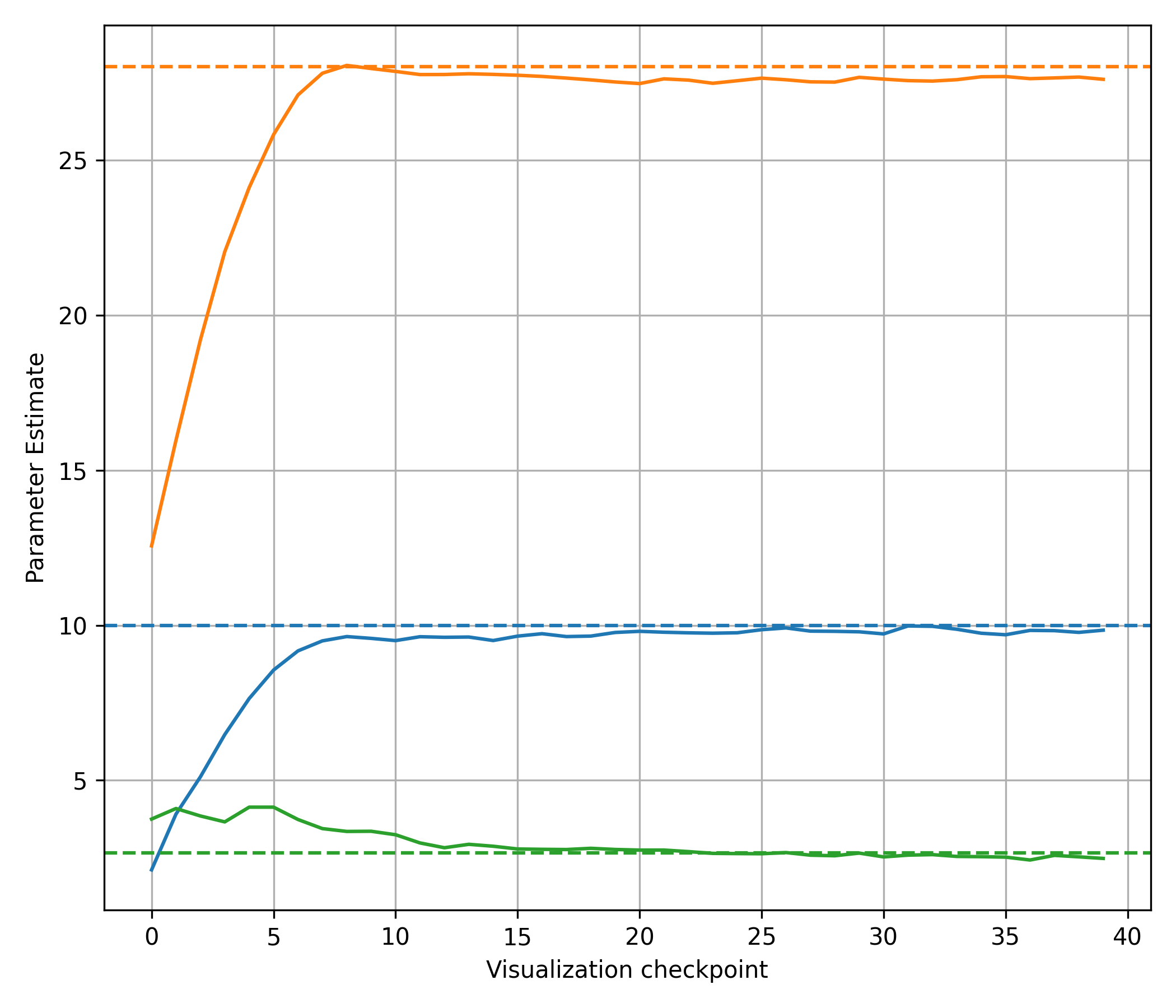}
        \caption{With sparse transformer.}
    \end{subfigure}
\caption{Lorenz-63 parameter inference. The first plot shows the validation error, where blue and orange correspond to with and without the sparse transformer, respectively. The second plot shows parameter estimates ($\sigma$: blue, $r$: orange, $\beta$: green) without the sparse transformer, and the third plot shows estimates with the sparse transformer; dashed lines indicate true values.  }
\label{fig:lorenz_experiments}
\end{figure}

From Fig.\,\ref{fig:lorenz_experiments}, we observe that the sparse transformer achieves more accurate parameter estimations while converging faster during the training process.

\item \textbf{Electrical impedance tomography.}  
Next, we consider the classical electrical impedance tomography (EIT) inverse problem governed by
\[
-\nabla \cdot (\sigma \nabla w) = 0 \quad \text{in } \Omega\,, 
\qquad 
\begin{cases}
    w = f\\ \frac{\partial w}{\partial \nu} = g
\end{cases}  \quad \text{on } \partial \Omega\,,
\]
where $\sigma(x)$ denotes the conductivity distribution to be recovered.  
For simplicity, we parameterize $\sigma$ using a Gaussian bump expansion
\[
\sigma(x) = \sum_{i=1}^N \sum_{j=1}^N 
c_{ij}\,\exp\!\left(-\tfrac{\|x - x_{ij}\|_2^2}{2\sigma^2}\right),
\]
with fixed grid points $x_{ij} \in [0,1]^2$ and unknown coefficients $\{c_{ij}\}$ to be inferred.  

We discretize and solve the forward problem using finite differences.  
For the inverse problem, we impose multiple boundary voltage patterns $f$ and record the corresponding Neumann data $g$, thereby generating datasets for conductivity recovery.  
Finally, we employ the proposed flow with the sparse transformer to approximate the posterior distribution $p(\sigma|g)$.

\begin{figure}[ht]
\centering 

\begin{subfigure}[t]{0.15\textwidth}
    \centering 
    \includegraphics[width=\textwidth]{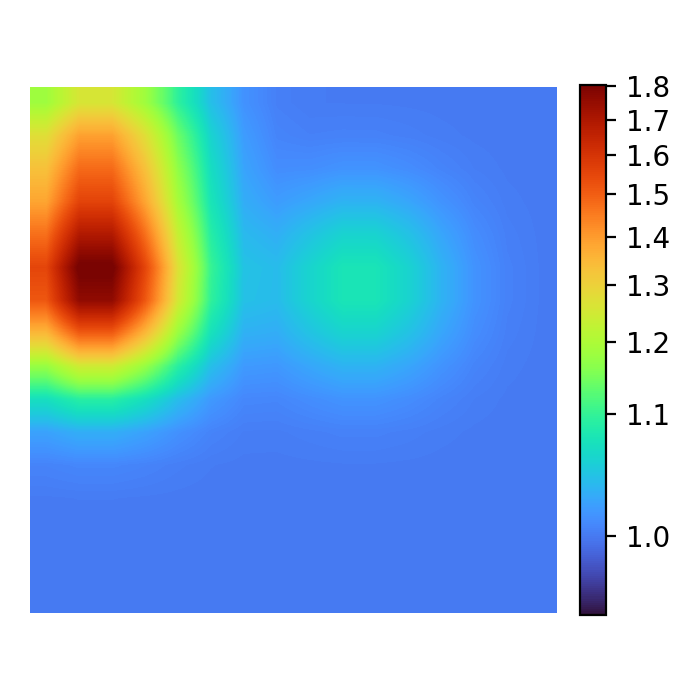}
\end{subfigure}
\hspace{0.02\textwidth}
\begin{subfigure}[t]{0.33\textwidth}
    \centering 
    \includegraphics[width=0.45\textwidth]{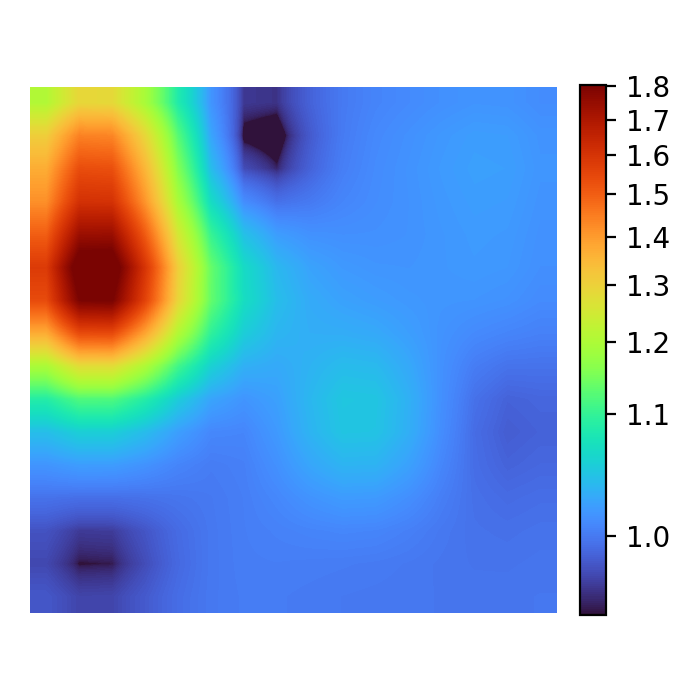}
    \centering 
    \includegraphics[width=0.45\textwidth]{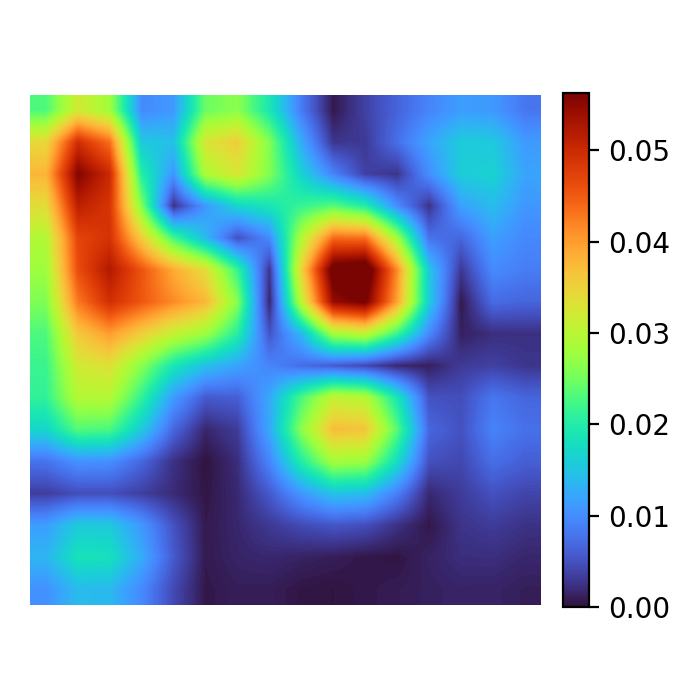}
\end{subfigure}
\hspace{0.02\textwidth}
\begin{subfigure}[t]{0.33\textwidth}
    \centering 
    \includegraphics[width=0.45\textwidth]{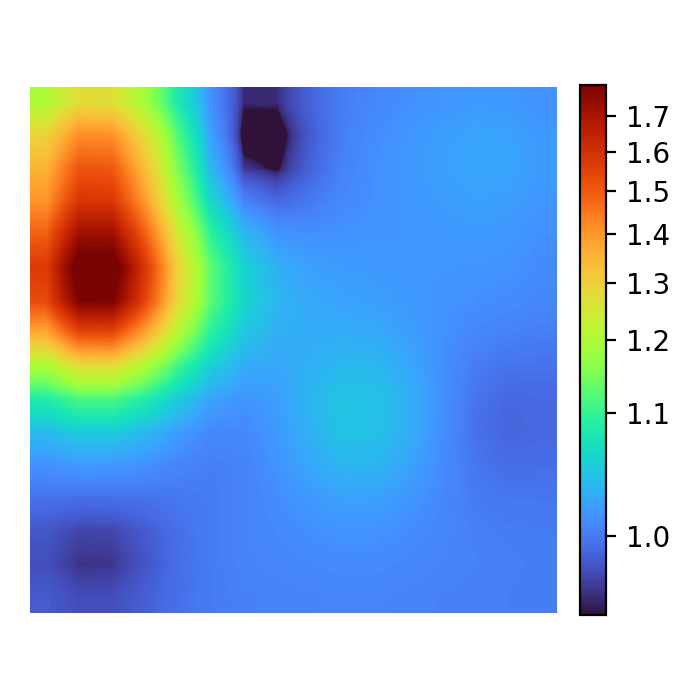}
    \centering 
    \includegraphics[width=0.45\textwidth]{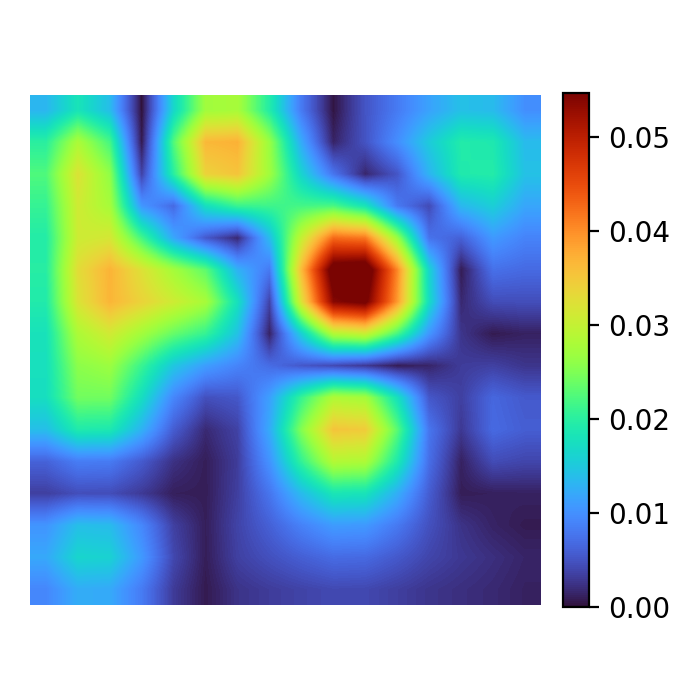}
\end{subfigure}

\begin{subfigure}[t]{0.15\textwidth}
    \centering 
    \includegraphics[width=\textwidth]{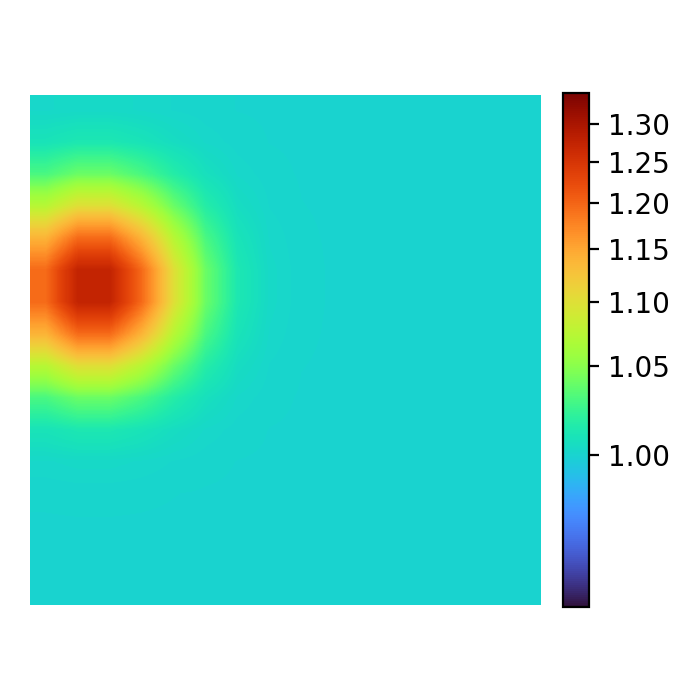}
\end{subfigure}
\hspace{0.02\textwidth}
\begin{subfigure}[t]{0.33\textwidth}
    \centering 
    \includegraphics[width=0.45\textwidth]{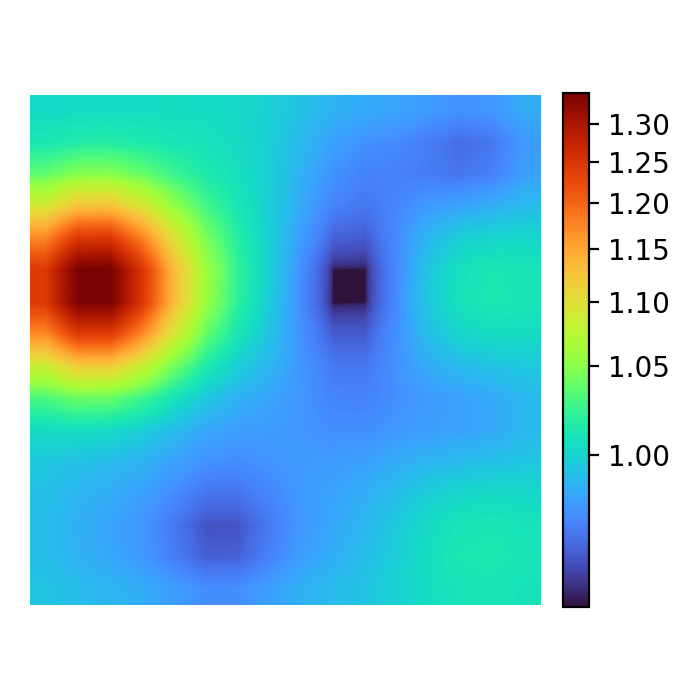}
    \centering 
    \includegraphics[width=0.45\textwidth]{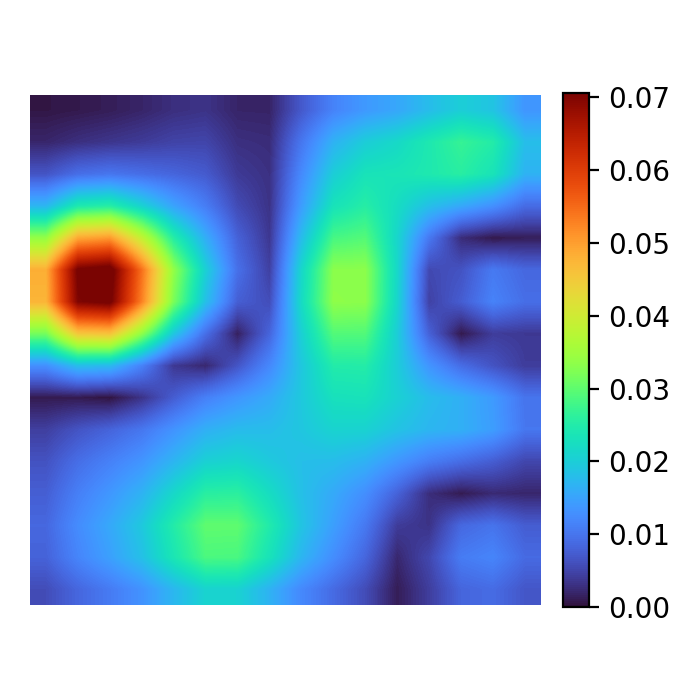}
\end{subfigure}
\hspace{0.02\textwidth}
\begin{subfigure}[t]{0.33\textwidth}
    \centering 
    \includegraphics[width=0.45\textwidth]{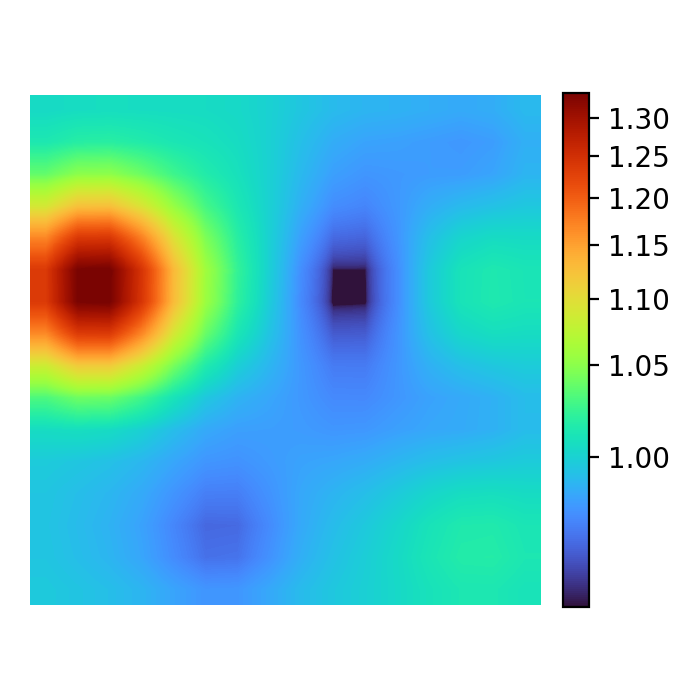}
    \centering 
    \includegraphics[width=0.45\textwidth]{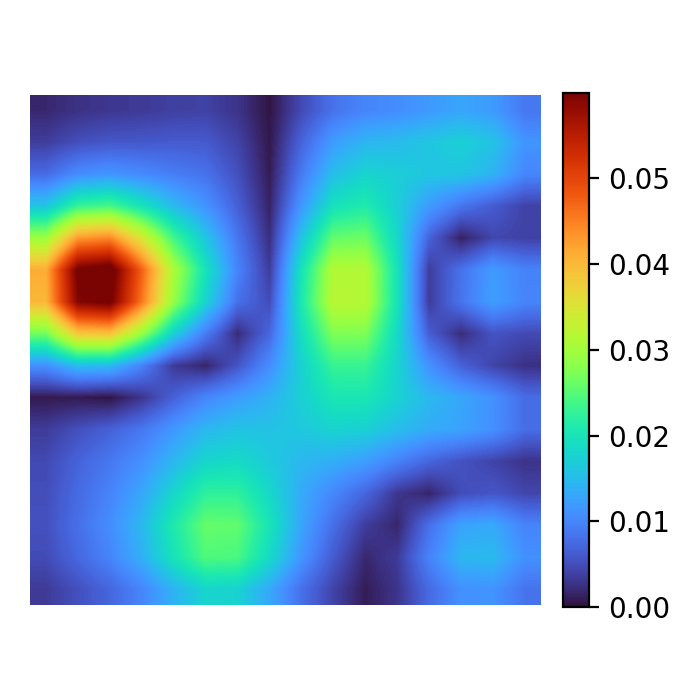}
\end{subfigure}

\begin{subfigure}[t]{0.15\textwidth}
    \centering 
    \includegraphics[width=\textwidth]{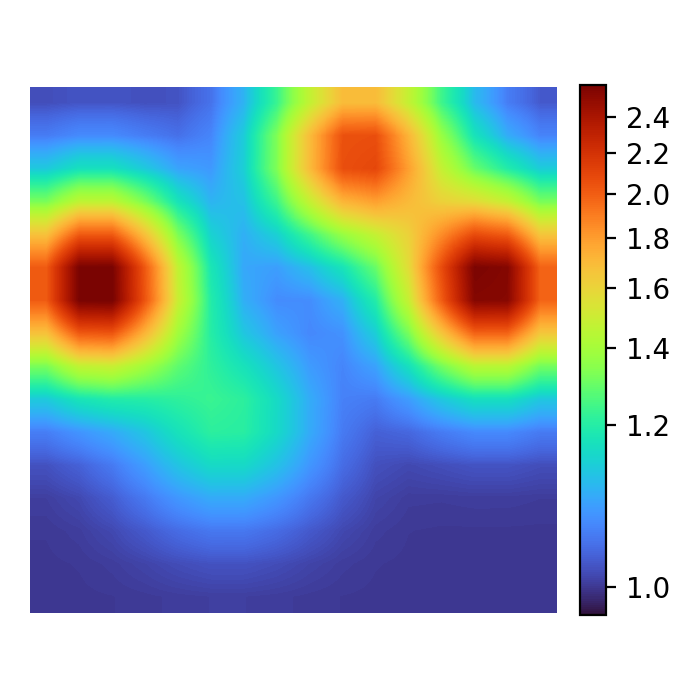}
    \caption*{True inclusion}
\end{subfigure}
\hspace{0.02\textwidth}
\begin{subfigure}[t]{0.33\textwidth}
    \centering 
    \includegraphics[width=0.45\textwidth]{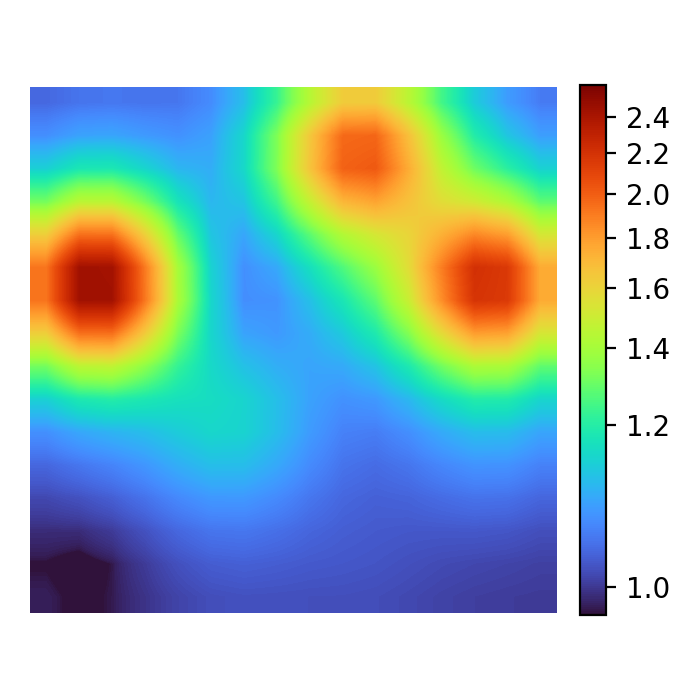}
    \centering 
    \includegraphics[width=0.45\textwidth]{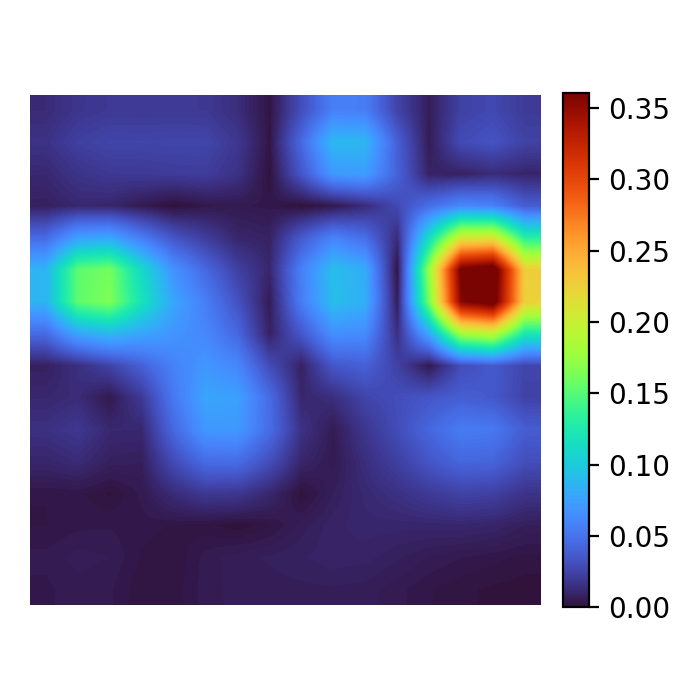}
    \caption*{Recovery and Error without sparse transformer}
\end{subfigure}
\hspace{0.02\textwidth}
\begin{subfigure}[t]{0.33\textwidth}
    \centering 
    \includegraphics[width=0.45\textwidth]{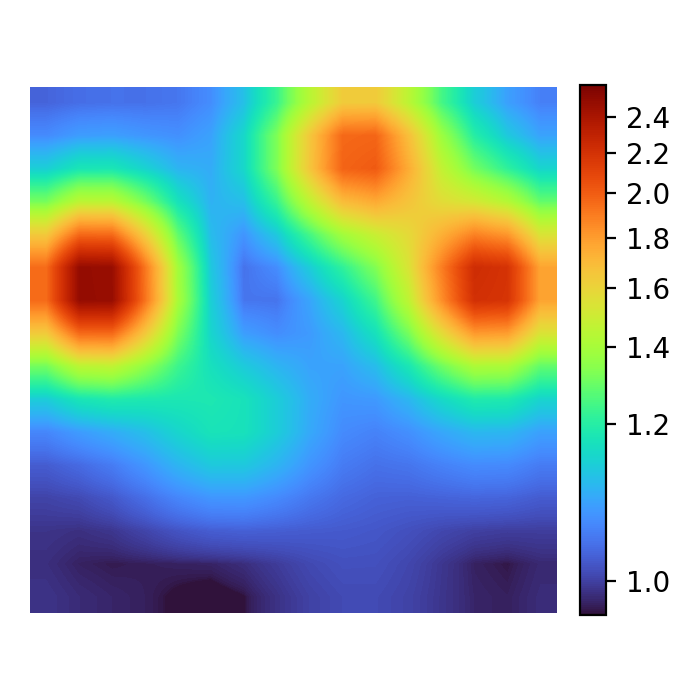}
    \centering 
    \includegraphics[width=0.45\textwidth]{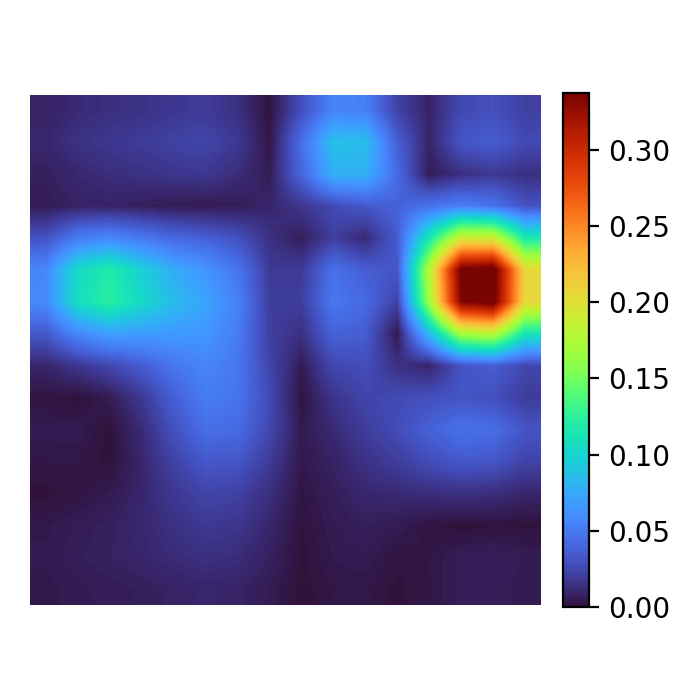}
    \caption*{Recovery and Error with sparse transformer}
\end{subfigure}

 \caption{EIT reconstructions for four different inclusions (rows). Each row shows the true conductivity (left), reconstructions and error maps without (middle two columns) and with sparse transformer (right two columns).}
    \label{fig:eit_lambda_comparison}
\end{figure}

For the three examples in Fig\,\ref{fig:eit_lambda_comparison}, the relative $L_1$ errors are as follows: 
\begin{itemize}
    \item Without sparse transformer: 
    $1.99\!\times\!10^{-2},\ 2.15\!\times\!10^{-2},\ 5.42\!\times\!10^{-2}$.
    \item With sparse transformer: 
    $1.70\!\times\!10^{-2},\ 1.87\!\times\!10^{-2},\ 4.86\!\times\!10^{-2}$.
\end{itemize}
These results show that the network with the sparse transformer consistently achieves more accurate reconstructions across various inclusions, as measured by the relative $L_1$ error.
\end{enumerate}

\subsection{Image generation for MNIST dataset.}  

We next evaluate the proposed model on the MNIST dataset. Following \cite{ot_flow}, we adopt an encoder–decoder architecture: the encoder $B:\mathbb{R}^{784}\to \mathbb{R}^d$ maps images to a latent representation, and the decoder $D:\mathbb{R}^d \to \mathbb{R}^{784}$ reconstructs images from this latent space. The encoder–decoder pair is trained independently to satisfy $D(B(x)) \approx x$. 

Our flow model is trained to transport the latent prior $\rho_0$ to the target distribution $\rho^*$. To generate new digit images, we sample from $\rho_0$, transport these samples to $\rho_T$ using the learned flow, and then decode them with $D$.

\begin{figure}[ht]
    \centering
    \includegraphics[width=0.2\linewidth,trim=0cm 0cm 0cm 0.8cm,clip]{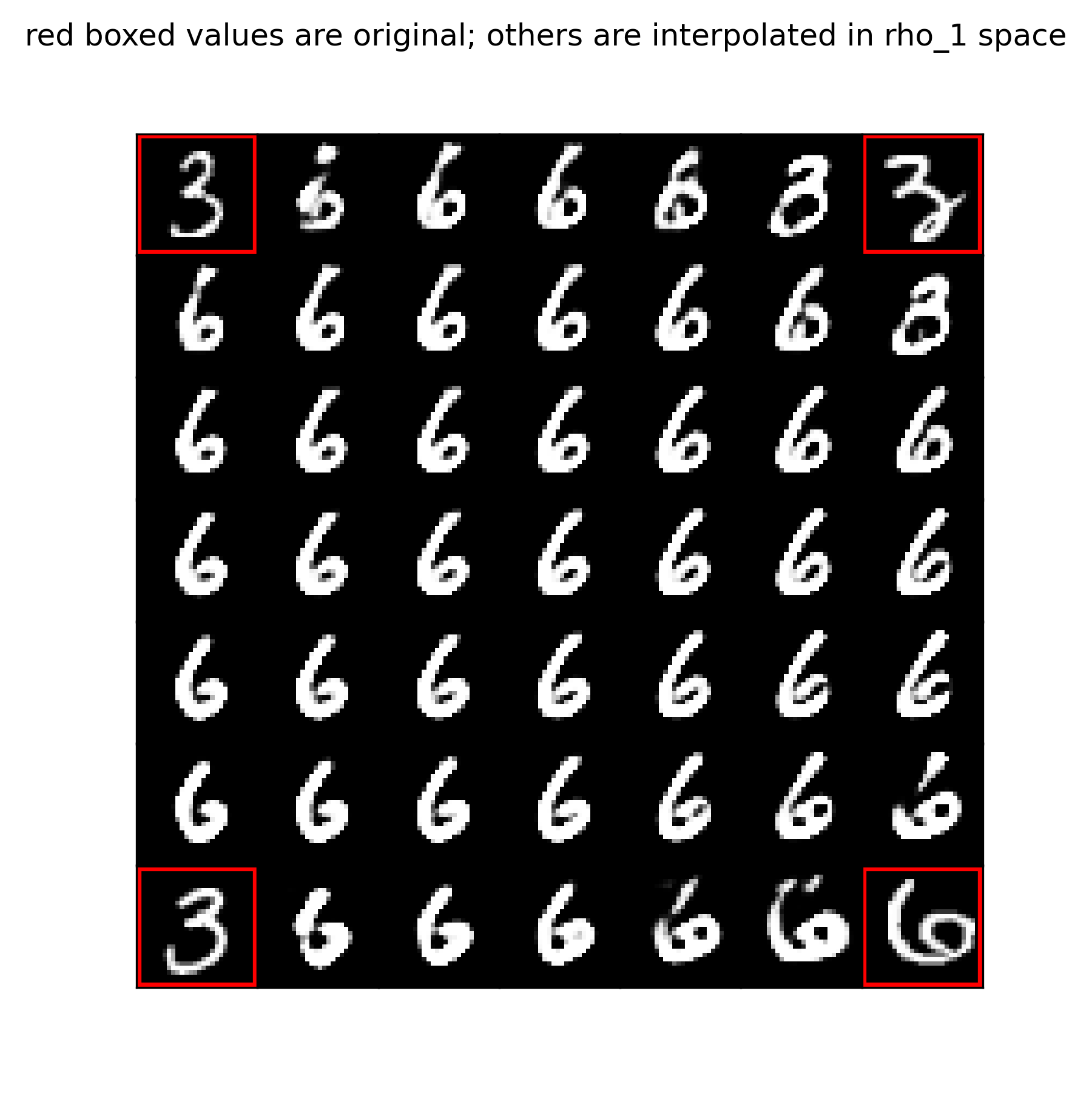}
    \includegraphics[width=0.2\linewidth,trim=0cm 0cm 0cm 0.8cm,clip]{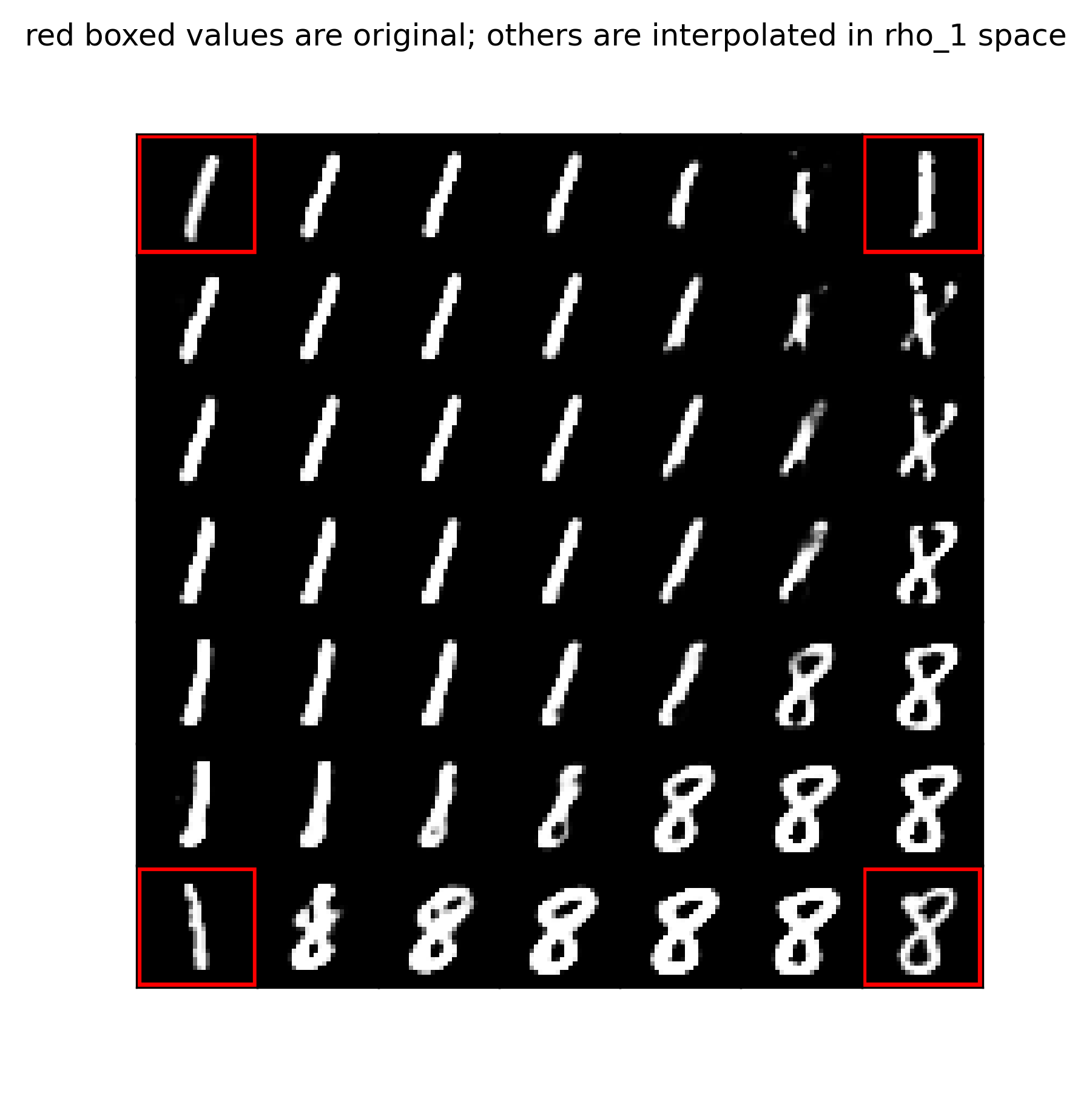}
    \includegraphics[width=0.2\linewidth,trim=0cm 0cm 0cm 0.8cm,clip]{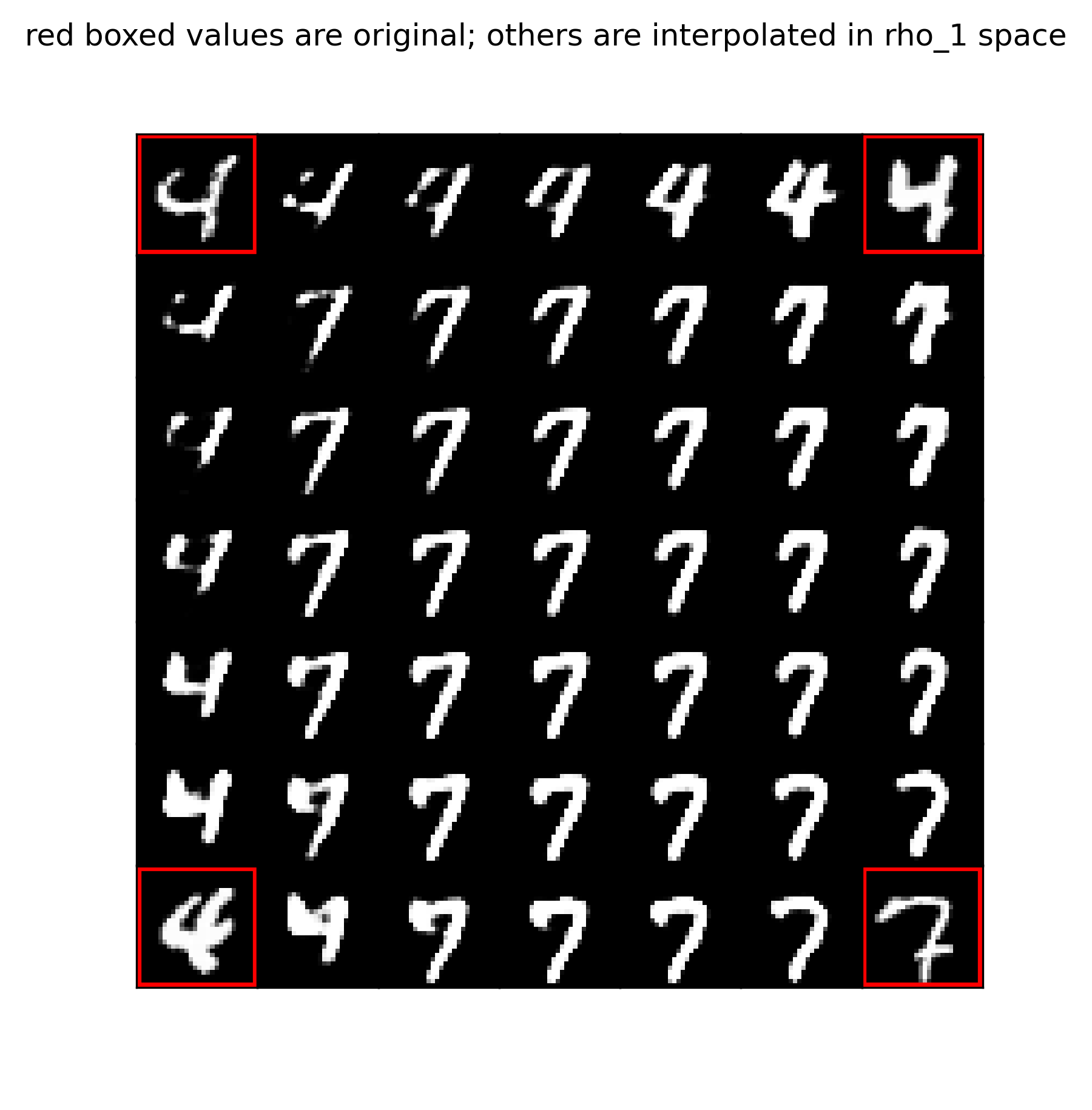}
    \includegraphics[width=0.2\linewidth,trim=0cm 0cm 0cm 0.8cm,clip]{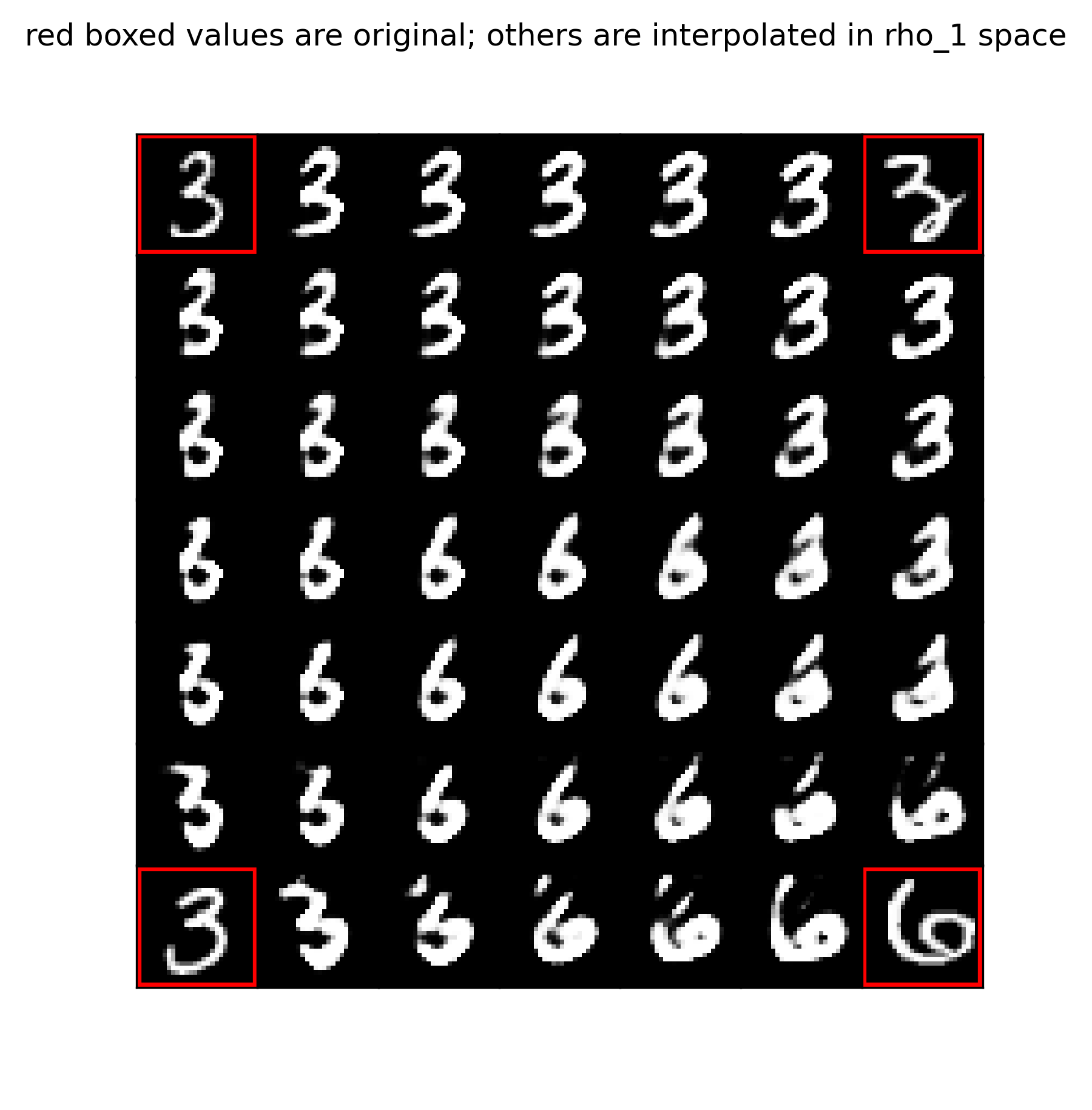}

    \caption{Digit interpolations on MNIST using the learned flow. The first example is generated without the sparse transformer, and the following three are generated with it.}
    \label{fig:mnist}
\end{figure}

Figure~\ref{fig:mnist} demonstrates that the sparse transformer effectively learns the MNIST distribution. Compared with the flow trained without the sparse transformer (first plot), which shows limited generalization, our model attains faster convergence and improved generalization quality.

\section{Summary and discussion}
In this work, we introduced a sparse transformer architecture derived from the regularized Wasserstein proximal operator with an $L_1$ prior distribution.
This formulation embeds problem-specific structural information, such as sparsity, directly into the network through the proximal mapping of the $L_1$ function, thereby improving both the accuracy and interpretability of the learned flow.

Theoretical analysis shows that the $L_1$ prior not only enforces sparsity in the learned representations but also accelerates convergence when the target distribution is sparse.
As established in Section~\ref{sec_long_time_analysis}, this mechanism yields faster KL decay and avoids gradient blow-up.
Experiments in generative modeling and Bayesian inverse problems further demonstrate that the proposed sparse transformer achieves more accurate recovery of target distributions and converges faster than standard neural ODE–based flows.

Beyond these results, our framework highlights how regularized Wasserstein proximal sampling methods serve as a unifying tool that bridges ideas from mean field games~\cite{ruthotto2020machine}, Wasserstein gradient flows~\cite{lee2024deep}, and optimal control formulations of transformer architectures~\cite{kan2025optimal}, thereby extending their reach within machine learning.
In future work, we plan to rigorously analyze the acceleration properties of RWPO-induced flows for more general target distributions under sparse or low-dimensional structures, and to extend this approach to other important priors, such as total variation or structured low-rank regularizations, within broader classes of generative AI and Bayesian inverse problems.

\begin{appendix}
\section{Postponed derivation and proof}
\label{sec_pf}
\begin{proof}[Proof of Lemma \ref{lem:rwpo}]
Introduce a Lagrange multiplier $\Phi(t,x)$ and consider the augmented Lagrangian
\begin{equation}
\label{Lagrang}
\mathcal{L} = \int_0^h \int \left( \tfrac{1}{2}\|v\|^2 \rho + \Phi\left( \partial_t \rho + \nabla\cdot(\rho v) - \beta^{-1}\Delta\rho \right)\right)\,dx\,dt + \int \psi(x) q(x)\,dx\,.
\end{equation}
The Euler-Lagrange equations yield
\begin{align}
\partial_t \rho + \nabla\cdot(\rho \nabla\Phi) &= \beta^{-1}\Delta\rho  \\
\partial_t \Phi + \tfrac{1}{2}\|\nabla \Phi\|^2 &= -\beta^{-1}\Delta \Phi\,,
\end{align}
with boundary conditions $\rho(0,x) = \rho_k(x)$ and $\Phi(h,x) = -\psi(x)$.  
Applying the Hopf-Cole transform
\[
\eta = \exp\!\left( \tfrac{\beta}{2}\Phi \right), \qquad \hat{\eta} = \frac{\rho}{\eta}\,,
\]
decouples the system into forward and backward heat equations with initial and terminal conditions $\eta(h,x) = \exp(-\tfrac{\beta}{2}\psi(x))$ and $\eta(0,x)\hat{\eta}(0,x) = \rho_k(x)$. Solving via convolution with the heat kernel leads directly to \eqref{closed_RWPO}.
\end{proof}

\begin{proof}[Proof of Proposition \ref{prop_KL_L1}]
Let $I(t)=I(\rho_t\|\rho^*)$. From \eqref{eq:KL-diss} and \eqref{KL-Fisher},
\begin{equation}
y'(t)= -\beta^{-1} I(t) - \lambda \mathcal D(\rho_t) \;\le\; -\beta^{-1} I(t) - \lambda\gamma \sqrt{I(t)}\,.
\end{equation}
By the log-Sobolev inequality, $I(t)\ge a\,y(t)$ with $a=2/C_{\mathrm{LS}}$. Hence
\begin{equation}
y'(t)\le -a y(t) - \lambda\gamma \sqrt{a}\,\sqrt{y(t)}\,.
\end{equation}
Consider the scalar comparison ODE $u'(t)=-a u(t) - \lambda\gamma \sqrt{a}\,\sqrt{u(t)}$, $u(0)=y(0)$. Setting $z(t)=\sqrt{u(t)}$ yields
$z'(t)=-(a/2)z(t)-(\lambda\gamma/2)\sqrt{a}$. Then, we have the explicit formula
\begin{equation}
z(t)=\Big(\sqrt{y(0)}+\tfrac{\lambda\gamma}{\sqrt{a}}\Big)e^{-\tfrac{a}{2}t}-\tfrac{\lambda\gamma}{\sqrt{a}}\,.
\end{equation}
Therefore $y(t)\le (z(t)_+)^2$, giving \eqref{eq:KL-upper}. 
\end{proof}

 \begin{proof}[Proof of Proposition \ref{prop:moment-ordering}]
Let $M_2^{\lambda}(t):=\int \|x\|_2^2 \rho_t^{\lambda}(x)\,dx$. 
Multiplying \eqref{eq:FP} by $\|x\|_2^2$ and integrating by parts gives
\begin{align}
\frac{d}{dt}M_2^{\lambda}(t)
&= \int \|x\|_2^2 \partial_t \rho_t^{\lambda}(x)\,dx \\
&= -2 \int x \cdot \nabla \phi(x)\, \rho_t^{\lambda}(x)\,dx
   - 2 \int x \cdot \nabla \psi^{\lambda}(x)\, \rho_t^{\lambda}(x)\,dx
   + 2 d \beta^{-1}\,.\notag
\end{align}
Since $\psi^{\lambda}(x)=\lambda\|x\|_1$ and $x\!\cdot\!\nabla\psi^{\lambda}(x)=\lambda\|x\|_1$, 
we obtain \eqref{eq:moment-evolution}, proving (i).

For (ii), consider $\nabla\phi(x)=-A x$ with $A\succeq0$. 
Then \eqref{eq:moment-evolution} becomes
\[
\frac{d}{dt} M_2^{\lambda}(t)
= -2\,\mathbb E\!\left[X_t^{\lambda\top} A X_t^{\lambda}\right]
  -2\lambda\,\mathbb E\!\left[\|X_t^{\lambda}\|_1\right]
  + 2 d \beta^{-1}.
\]
Under a synchronous coupling (shared noise and identical initial data), 
the drift $-\lambda\,\mathrm{sign}(X_t^{\lambda})$ exerts a stronger pull toward the origin for larger~$\lambda$, 
implying that $|X_t^{\lambda_1}|\le |X_t^{\lambda_2}|$ holds coordinatewise almost surely.
Taking expectations yields $M_2^{\lambda_1}(t)\le M_2^{\lambda_2}(t)$ for all $t\ge0$, 
which proves \eqref{eq:moment-ordering}.
\end{proof}

\noindent\textit{Derivation of \eqref{KKT_system}.} 
Let $\mathcal{I}$ denote the objective functional in \eqref{eqn_uncon_1}. 
The corresponding  KKT conditions are obtained by setting the first variations with respect to $\rho$, $\Phi$, and $\tilde{\phi}$ to zero:
\begin{equation}
\begin{cases}
\dfrac{\delta \mathcal{I}}{\delta \rho} 
= \tfrac{1}{2}\|\nabla \tilde{\phi}\|_2^2 
- \partial_t \Phi 
- \nabla \Phi \!\cdot\! \nabla \tilde{\phi}  
- \beta^{-1}\Delta \Phi = 0\,, \\ 
\dfrac{\delta \mathcal{I}}{\delta \Phi} 
= \partial_t \rho 
+ \nabla \!\cdot\! (\rho \nabla \tilde{\phi})  
- \beta^{-1}\Delta \rho = 0\,, \\
\dfrac{\delta \mathcal{I}}{\delta (\nabla \tilde{\phi})} 
= \rho(\nabla \tilde{\phi} - \nabla \Phi) = 0\,, \\\
\dfrac{\delta \mathcal{I}}{\delta \rho_T} 
= -\Phi_T - \tfrac{\rho^*}{\rho_T} = 0\,,
\end{cases}
\end{equation}
which leads to \eqref{KKT_system}.



\section{Details of numerical experiments}
\label{sec_detail}

We summarize the parameters used for model initialization and training in Table~\ref{tab:parameters}.  

{\begin{table}[ht]
\footnotesize
\centering
\caption{Model and training parameters for different examples.}
\begin{tabular}{|l|c|c|c|c|c|c|c|c|}
\hline
Example & $d$ & $n_t$ & $c$ & width & $n_{\mathrm{iters}}$ & batch size & learning rate & noise \\
\hline
Benchmark & 4 & 48 & 0.1 & 48 & 6001 & 4096 & $1\times 10^{-2} \to 1\times 10^{-6}$ & - \\
MNIST & 64 & 48 & 5 & 128 & 12001 & 1024 & $1\times 10^{-3} \to 1\times 10^{-6}$ & - \\
Elliptic PDE & 128 & 48 & 0.1 & 128 & 12001 & 1024 & $1\times 10^{-3} \to 1\times 10^{-6}$ & $1\times 10^{-2}$ \\
Lorenz system & 3 & 48 & 0.1 & 128 & 12001 & 1024 & $1\times 10^{-3} \to 1\times 10^{-6}$ & $1\times 10^{-2}$ \\
EIT & 32 & 48 & 0.1 & 512 & 12001 & 512 & $5\times 10^{-3} \to 1\times 10^{-6}$ & $1\times 10^{-3}$ \\
\hline
\end{tabular}
\label{tab:parameters}
\end{table}}

\noindent
\textbf{Explanation:}  
Here, $d$ is the problem dimension or latent space dimension; $n_t$ is the number of time steps for the token dynamics \eqref{eqn_particle_intro}; $c$ is the coefficients for the HJB regularizer in \eqref{objective}; {width} is the number of neurons per layer; $n_{\mathrm{iters}}$ is the total number of training iterations; {batch size} is the number of samples per iteration; the learning rate shows the initial and final values; {noise} indicates the noise level in the training data. Moreover, we set $\lambda = 2$ and $\beta = 1$ for the initialization.

\end{appendix}

\bibliographystyle{plain}
\bibliography{ref}
\end{document}